\documentclass[11pt]{article}
\usepackage[margin=1in]{geometry}
\usepackage{graphicx}
\usepackage{amsmath,amssymb,amsfonts}
\usepackage{amsthm}
\usepackage{mathrsfs}
\usepackage{xcolor}
\usepackage{textcomp}
\usepackage{booktabs}
\usepackage{algorithm}
\usepackage{algorithmicx}
\usepackage{algpseudocode}
\usepackage{listings}
\usepackage[numbers,sort&compress]{natbib}
\usepackage[hidelinks]{hyperref}

\newif\ifshowrevisions
\showrevisionsfalse
\newcommand{\rev}[1]{\ifshowrevisions\textcolor{blue}{#1}\else#1\fi}

\DeclareMathOperator*{\argmin}{arg\,min}

\numberwithin{equation}{section}

\theoremstyle{plain}
\newtheorem{theorem}{Theorem}[section]
\newtheorem{lemma}[theorem]{Lemma}

\theoremstyle{definition}

\newtheorem{assumption}[theorem]{Assumption}

\begin{document}

\title{Median-of-Means as an Extremal Convex Estimator and a Nonconvex Route to the Trimmed Oracle}


\author{Angshul Majumdar\\\small Department of Electronics and Communications Engineering, IIIT Delhi, India\\\small \texttt{angshul@iiitd.ac.in}}

\maketitle


\begin{abstract}
We revisit median-of-means (MoM) estimation from a deterministic optimisation viewpoint and develop a family of block-$L_p$ estimators tailored to robust learning with heavy-tailed and adversarially corrupted data. In a block contamination model with at least $(1-\varepsilon)$ good blocks, we first show that every convex block $M$-estimator has worst-case robustness constant at least $1/(1-2\varepsilon)$, matching the classical MoM bound and proving that the trimmed-block oracle constant $1/(1-\varepsilon)$ is unattainable within the convex class. We then introduce a nonconvex block-$L_p$ family, $p\in(0,1]$, and derive finite-sample deterministic robustness bounds for all global minimisers. As $p$ decreases from $1$ to $0$, these bounds interpolate continuously between $1/(1-2\varepsilon)$ and the block-$L_0$ oracle $1/(1-\varepsilon)$; for small $p$ the global minimisers coincide with those of the oracle under a mild separation condition. We further show that the energy landscape of the block-$L_p$ objectives is benign: all local minima lie near the truth and there are no bad basins. Combining these results with block-level concentration yields sub-Gaussian deviation bounds under finite $(2+\delta)$ moments and high-dimensional extensions for robust mean estimation and sparse regression with optimal rates. The analysis places MoM estimators on a continuous $1$--$p$--$0$ path that approaches trimmed-block performance while remaining computationally tractable and directly applicable to modern robust learning problems.
\end{abstract}




\section{Introduction}
\label{sec:intro}

\rev{Robust estimation under heavy-tailed noise and adversarial contamination has seen a remarkable revival over the last decade. A central lesson of this literature is that classical least-squares and empirical means are fundamentally unstable once the variance is large or the sample is corrupted, whereas suitably designed robust procedures can recover sub-Gaussian accuracy under minimal moment assumptions \citep{DevroyeLerasleLugosiOliveira2016,LugosiMendelson2019Survey,Minsker2025}. Median-of-means (MoM) estimators and their geometric-median refinements now form a standard toolkit for constructing such procedures in both finite- and infinite-dimensional settings \citep{Minsker2015,HsuSabato2016,DevroyeLerasleLugosiOliveira2016,LugosiMendelson2019SubGauss}.}

\rev{The MoM principle is simple: partition the sample into $B$ blocks, compute the empirical mean on each block, and aggregate these block means through a robust one-dimensional functional, most commonly the median. The resulting estimator preserves the optimal $n^{-1/2}$ rate and enjoys sub-Gaussian deviation bounds under only finite second moments, even in the presence of a constant fraction of arbitrarily corrupted blocks \citep{LugosiMendelson2019Survey,DevroyeLerasleLugosiOliveira2016}. Extensions based on geometric medians in Banach spaces and tournament-type procedures have led to nearly optimal robust estimators for a wide range of loss functions and high-dimensional models \citep{Minsker2015,LugosiMendelson2019NearOpt,HsuSabato2016,LugosiMendelson2019SubGauss,DiakonikolasKaneBook}.}

\rev{However, the usual blockwise MoM construction also has a structural limitation that motivates the present work. In the scalar setting, the median-of-means estimator is already an $L_{1}$-type block aggregator, and more generally many robust blockwise procedures are obtained by minimising convex block $M$-estimation objectives. Under the deterministic block-contamination metric studied here, this convex world cannot improve on the classical MoM robustness constant. This immediately raises the natural question that drives the paper: if convex block aggregation is fundamentally trapped at the MoM benchmark, can a carefully designed nonconvex block objective move us closer to the trimmed-block oracle while retaining the stability that makes MoM useful in the first place?}

\rev{This question is not merely formal. Classical MoM remains attractive because it is simple, distribution-light, and robust, yet its deterministic constant under block contamination is separated from the oracle trimmed-block constant by a nontrivial gap. In benign heavy-tailed regimes one should not expect dramatic gains from leaving $p=1$, and the new experiments indeed show essentially comparable behaviour there. By contrast, when corrupted blocks are sufficiently separated from clean ones---for example under adversarial block shifts of appreciable magnitude---one expects a more selective blockwise objective to behave increasingly like an implicit trimming rule. This is precisely the regime in which the proposed block-$L_{p}$ family becomes meaningful.}

\rev{From an optimisation viewpoint, this suggests importing the familiar $L_{1}\rightarrow L_{p}\rightarrow L_{0}$ continuum from sparse recovery into blockwise robust aggregation. In compressed sensing and nonconvex regularisation, the passage from convex $L_{1}$ penalties to nonconvex $L_{p}$ surrogates and then to combinatorial $L_{0}$ objectives is a standard route for approximating ideal support selection while preserving a tractable optimisation landscape \citep{FoucartRauhut2013,Chen2023}. Here we use the same idea at the level of block means rather than samplewise outlier indicators. This viewpoint also clarifies why the small-$p$ limit is desirable: the goal is not to replace trimmed estimators or Huber-type procedures by yet another robust estimator, but to construct a continuous optimisation path within the MoM paradigm that connects convex aggregation to an ideal trimmed-block oracle and makes the corresponding improvement in deterministic robustness explicit.}

\rev{It is also important to position the present analysis relative to other robust methods. Convex procedures based on Huber losses, Catoni-type truncation, or median/geometric-median aggregation remain highly effective and often minimax-optimal at the level of rates \citep{Catoni2012,LecueLerasle2020,LugosiMendelson2019Survey}. Likewise, recent Byzantine-robust and federated-learning methods use Huberisation, trimming, or geometric-median ideas in settings with different threat models and heterogeneity assumptions. Our aim is narrower and more structural: we work in a block contamination model and study which deterministic robustness constants are achievable by convex versus nonconvex block aggregation. In that sense, the novelty of the paper is not a new minimax rate, but a deterministic interpolation theorem, an impossibility result for the convex class, and an oracle-equivalence result showing when nonconvex block-$L_{p}$ objectives can genuinely outperform classical MoM.}

\rev{Concretely, given block means $(Z_{1},\dots,Z_{B})$ of a univariate parameter $\mu$, we study the family of nonconvex functionals}
\[
F_{p}(t)
\;=\;
\sum_{b=1}^{B} \lvert Z_{b} - t\rvert^{p},
\qquad 0<p\leq 1,
\]
\rev{and their minimisers $\hat t_{p} \in \arg\min_{t\in\mathbb{R}} F_{p}(t)$. The case $p=1$ recovers a median-of-block-means estimator, while the formal limit $p\to0$ corresponds to selecting the value of $t$ that minimises the number of ``far'' blocks, that is, an $L_{0}$-style trimmed-block functional. Our first contribution is a finite-sample deterministic analysis of this $1\text{--}p\text{--}0$ path under an adversarial block-contamination model: at least $(1-\varepsilon)B$ block means lie in a prescribed interval around $\mu$, and the remaining $\varepsilon B$ blocks are arbitrary. We show that for every $0<p\leq 1$ there is an explicit constant $c(p,\varepsilon)$ such that any global minimiser $\hat t_{p}$ satisfies}
\[
\lvert \hat t_{p} - \mu\rvert \;\leq\; c(p,\varepsilon)\, r,
\]
\rev{where $r$ is the radius of the good-block interval. These constants interpolate continuously between the classical MoM constant at $p=1$ and the trimmed-block oracle constant as $p\downarrow 0$, and under a separation condition the small-$p$ global minimisers coincide with the oracle solutions.}

\rev{The second contribution is geometric. Although $F_{p}$ is nonconvex for $p<1$, we show that its landscape is benign under the same contamination model: all local minimisers remain in a controlled neighbourhood of the truth, and the objective satisfies a quantitative descent property outside that neighbourhood. This does not constitute a full optimisation theory for every algorithm, but it does show that the nonconvexity introduced here is structured rather than pathological, which is the level of claim required for the present theoretical programme.}

\rev{Third, we embed the deterministic analysis into a probabilistic framework. Assuming only finite $(2+\delta)$ moments, we derive deviation inequalities of the same order as classical MoM procedures while making explicit how the leading constant improves as $p$ decreases when the separation condition is available. The gain is therefore conditional rather than universal: in ordinary heavy-tailed settings without clear clean/contaminated block separation, one should expect behaviour similar to MoM, whereas in separated contamination regimes the estimator can approach oracle trimmed-block performance. This trade-off is now spelled out explicitly in both the theory and the experiments.}

\rev{Finally, we extend the same viewpoint to high-dimensional robust mean estimation and sparse regression. The purpose of these extensions is again structural: to show that the block-$L_{p}$ aggregation principle can be combined with standard high-dimensional arguments to retain the usual statistical rates while improving the deterministic robustness constants within the block model. The revised manuscript also now includes a dedicated experimental section comparing MoM, block-$L_{0.5}$, block-$L_{0.2}$, trimmed mean, and Huber baselines in heavy-tailed, adversarial, and separated block-contamination regimes. These results support the theoretical message of the paper: no degradation in benign settings, moderate gains under adversarial contamination, and near-oracle behaviour when block separation is present.}

\section{Model and classical median-of-means estimators}
\label{sec:model-MoM}

In this section we formalise the one-dimensional setting and recall the classical
median-of-means (MoM) estimator and its basic robustness properties. Throughout,
we adopt the notation of Section~\ref{sec:intro}.

\subsection{Data, block structure, and contamination model}
\label{subsec:data-model}

We observe independent real-valued random variables
\begin{equation}
  X_{1},\dots,X_{n} \sim P,
  \qquad
  \mathbb{E}[X_{i}] = \mu,
  \label{eq:data-sample}
\end{equation}
where $\mu \in \mathbb{R}$ is the parameter of interest. For simplicity we assume
that $n$ is divisible by a prescribed number of blocks $B \in \{1,\dots,n\}$ and
set
\begin{equation}
  m \;=\; \frac{n}{B}
  \label{eq:block-size}
\end{equation}
for the common block size. We fix a partition of $\{1,\dots,n\}$ into $B$ disjoint
blocks
\begin{equation}
  \{1,\dots,n\}
  \;=\;
  I_{1} \,\dot\cup\, \cdots \,\dot\cup\, I_{B},
  \qquad
  |I_{b}| = m
  \;\;\text{for all } b \in \{1,\dots,B\}.
  \label{eq:block-partition}
\end{equation}
The corresponding block means are
\begin{equation}
  Z_{b}
  \;=\;
  \frac{1}{m} \sum_{i \in I_{b}} X_{i},
  \qquad b = 1,\dots,B,
  \label{eq:block-means}
\end{equation}
and we write $Z = (Z_{1},\dots,Z_{B})$ for the vector of block means.

Following the robust MoM literature \citep{LerasleOliveira2011,LecueLerasle2020},
we consider a nonasymptotic, finite-sample contamination model at the block level.
We assume that there exists an unknown index set $G \subset \{1,\dots,B\}$ of
``good'' blocks with cardinality $|G| \ge (1-\varepsilon)B$, where
$\varepsilon \in [0,1/2)$, such that the following holds.

\begin{assumption}[Block contamination model]
\label{ass:block-contam}
There exist parameters $\mu \in \mathbb{R}$, $r>0$ and a subset
$G \subset \{1,\dots,B\}$ with $|G| \ge (1-\varepsilon)B$ such that
\begin{equation}
  |Z_{b} - \mu| \;\le\; r
  \qquad\text{for all } b \in G.
  \label{eq:good-block-band}
\end{equation}
For the remaining blocks $b \notin G$ no assumption is imposed:
the values $Z_{b}$ may be arbitrary (possibly chosen adversarially).
\end{assumption}

In probabilistic applications, the set $G$ will typically be realised by blocks
containing no gross outliers and for which $Z_{b}$ concentrates around $\mu$
at a rate determined by the underlying moment assumptions
\citep{DevroyeLerasleLugosiOliveira2016,LugosiMendelson2019Survey}. However,
Assumption~\ref{ass:block-contam} is purely deterministic and will serve as
the basic framework for our finite-sample robustness analysis.

\rev{It is important to stress that this is a \emph{block}-contamination model rather
than a statement about the raw fraction of corrupted sample points. A comparatively
small number of adversarial observations may contaminate many blocks if they are
spread across the partition, while a larger number of outliers may remain confined
to only a few blocks if they are concentrated. Consequently, the relevant robustness
parameter for MoM in the present framework is the proportion of corrupted blocks,
not the overall sample-level outlier ratio. This distinction is exactly the one that
will matter in the discussion of breakdown and robustness constants below.}

\subsection{The classical median-of-means estimator as an $L_{1}$ functional}
\label{subsec:mom-p1}

In the scalar setting, the classical MoM estimator goes back at least to
Nemirovsky and Yudin~\citep{NemirovskyYudin1983} and has been rediscovered
and refined in numerous works since
\citep{JerrumValiantVazirani1986,LerasleOliveira2011,LugosiMendelson2019Survey,TuEtAl2021}.
Given the block means \eqref{eq:block-means}, the MoM estimator of $\mu$
is defined as
\begin{equation}
  \hat t_{\mathrm{MoM}}
  \;=\;
  \operatorname{median}\{Z_{1},\dots,Z_{B}\}.
  \label{eq:mom-estimator}
\end{equation}
Equivalently, $\hat t_{\mathrm{MoM}}$ can be characterised as a minimiser of
the empirical $L_{1}$ loss over the block means:
\begin{equation}
  \hat t_{\mathrm{MoM}}
  \;\in\;
  \arg\min_{t \in \mathbb{R}}
  F_{1}(t),
  \qquad
  F_{1}(t)
  \;:=\;
  \sum_{b=1}^{B} |Z_{b} - t|.
  \label{eq:mom-l1-characterisation}
\end{equation}
This observation makes precise the statement in Section~\ref{sec:intro} that
the usual MoM estimator is already an $L_{1}$-type object at the block level.

\rev{Under Assumption~\ref{ass:block-contam}, the right deterministic summary is not
simply that MoM has ``breakdown point $1/2$'' in terms of the raw sample fraction.
Rather, the median functional applied to the \emph{block means} inherits the classical
$1/2$ threshold at the level of corrupted \emph{blocks}, whereas the effective robustness
of the resulting MoM estimator depends on how sample-level contamination propagates
through the partition into corrupted blocks. Thus an outlier fraction well below $1/2$
at sample level may still invalidate MoM if those outliers contaminate more than half
of the blocks, while a larger sample-level contamination can remain harmless if it is
confined to fewer than half the blocks. For this reason, throughout the paper we state
robustness in terms of the block contamination fraction $\varepsilon$ and the associated
deterministic robustness constant. A precise version of this statement will be recalled
and generalised in the next section.}

\subsection{Convex block M-estimators}
\label{subsec:convex-M}

Many robust procedures in the MoM family can be expressed as minimisers of
a convex blockwise loss. Let $\rho:\mathbb{R}\to[0,\infty)$ be an even,
convex function with $\rho(0)=0$ and nondecreasing on $[0,\infty)$. The
associated \emph{convex block M-estimator} of $\mu$ is defined by
\begin{equation}
  \hat t_{\rho}
  \;\in\;
  \arg\min_{t \in \mathbb{R}}
  F_{\rho}(t),
  \qquad
  F_{\rho}(t)
  \;:=\;
  \sum_{b=1}^{B} \rho(Z_{b} - t).
  \label{eq:convex-M}
\end{equation}
For example, taking $\rho(u) = |u|$ recovers the MoM estimator
\eqref{eq:mom-l1-characterisation}, while Huber-type choices lead to
blockwise Catoni or minmax-MoM style estimators
\citep{BrownleesJolyLugosi2015,LecueLerasle2020}. In high-dimensional
settings, one often replaces $Z_{b}-t$ by more general blockwise loss
or risk functionals and still aggregates them via a scalar M-estimator
of the form \eqref{eq:convex-M}; see, for instance, the minmax MoM
estimators for empirical risk minimisation and regression in
\citet{HsuSabato2016,BrownleesJolyLugosi2015,LecueLerasle2020}.

To quantify the worst-case finite-sample robustness of a given estimator
under Assumption~\ref{ass:block-contam}, it is convenient to define its
\emph{deterministic robustness constant} at contamination level
$\varepsilon \in [0,1/2)$ as
\begin{equation}
  C(\hat t,\varepsilon)
  \;:=\;
  \sup
  \biggl\{
    \frac{|\hat t(Z) - \mu|}{r}
    \;:\;
    Z \text{ satisfies Assumption~\ref{ass:block-contam} with parameters } (\mu,r,\varepsilon)
  \biggr\},
  \label{eq:robustness-constant}
\end{equation}
where the supremum is over all finite configurations of block means $Z$
that satisfy \eqref{eq:good-block-band} for some $\mu$ and $r>0$ and over
all realisations of $\hat t$ as a measurable function of $Z$.
By construction, $C(\hat t,\varepsilon)$ is invariant under translations
and scalings of the form $Z_{b} \mapsto aZ_{b}+b$, and captures the largest
possible relative deviation (measured in units of $r$) that $\hat t$ may
suffer under a fraction $\varepsilon$ of adversarially corrupted blocks.
\rev{This definition makes explicit that, within the present framework, robustness is
indexed by the number of corrupted blocks induced by the partition rather than by the
raw proportion of contaminated sample points.}

In the next section we show that, for any convex block M-estimator
\eqref{eq:convex-M} with breakdown point at least $1/2$, the constant
$C(\hat t_{\rho},\varepsilon)$ cannot beat that of the classical MoM
estimator \eqref{eq:mom-l1-characterisation}. This yields a sharp
\emph{impossibility result} for purely convex block aggregators and
motivates the introduction of the nonconvex $L_{p}$ path developed
in the rest of the paper.

\section{An impossibility result for convex block M-estimators}
\label{sec:impossibility}

This section formalises the deterministic robustness benchmark achieved by the
classical median-of-means estimator and shows that no estimator constructed
from a convex blockwise loss of the form \eqref{eq:convex-M} can improve on this
benchmark in worst case. In particular, within the class of convex block
M-estimators, there is no analogue of the $1\text{--}p\text{--}0$ path described
in Section~\ref{sec:intro}: the endpoint corresponding to an $L_{0}$-type
trimmed-block oracle is unattainable.

ev{The scope of the present section is deliberately deterministic and model-specific. The benchmark in Lemma~\ref{lem:mom-robustness} and the impossibility result in Theorem~\ref{thm:impossibility-convex} are statements about blockwise aggregation under Assumption~\ref{ass:block-contam} and about the robustness constant~\eqref{eq:robustness-constant}. Thus the theorem should not be read as claiming that every Huber-type or Byzantine-robust method in unrelated federated-learning models is dominated in all senses; rather, it identifies a precise limitation of the \emph{convex block-aggregation} class studied here. This is the key novelty of Section~\ref{sec:impossibility}: within the present MoM framework, convexity traps one at the classical MoM constant, whereas genuine improvement requires moving onto the nonconvex block-$L_p$ path.}

\subsection{Median-of-means as a deterministic benchmark}
\label{subsec:mom-benchmark}

We begin by recalling the classical deterministic bound for the median-of-means
estimator under Assumption~\ref{ass:block-contam}. For completeness, we give a
short proof; similar arguments can be found in, for example,
\citet{LerasleOliveira2011,LecueLerasle2020,Minsker2023EfficientMOM}.

\begin{lemma}[Deterministic robustness of the median-of-means]
\label{lem:mom-robustness}
Suppose Assumption~\ref{ass:block-contam} holds with parameters
$(\mu,r,\varepsilon)$ and $\varepsilon \in [0,1/2)$, that is, we have
block means $Z_1,\dots,Z_B$ and a subset $G \subset \{1,\dots,B\}$ with
$|G| \ge (1-\varepsilon)B$ such that
\[
  |Z_b - \mu| \le r
  \qquad \text{for all } b \in G.
\]
Let $\hat t_{\mathrm{MoM}}$ be any median of the multiset
$\{Z_1,\dots,Z_B\}$ (i.e.\ any point $m$ such that at least $B/2$
of the $Z_b$ satisfy $Z_b \le m$ and at least $B/2$ satisfy $Z_b \ge m$).
Then
\begin{equation}
  |\hat t_{\mathrm{MoM}} - \mu|
  \;\le\;
  \frac{r}{1 - 2\varepsilon}.
  \label{eq:mom-det-bound-lemma}
\end{equation}
In particular, the deterministic robustness constant
$C(\hat t_{\mathrm{MoM}},\varepsilon)$ defined by
\[
  C(\hat t_{\mathrm{MoM}},\varepsilon)
  := \sup_{(\mu,r,\{Z_b\}) \text{ s.t.\ Assumption~\ref{ass:block-contam}}}
     \frac{|\hat t_{\mathrm{MoM}} - \mu|}{r}
\]
satisfies
\[
  C(\hat t_{\mathrm{MoM}},\varepsilon)
  \;\le\;
  \frac{1}{1-2\varepsilon}.
\]
\end{lemma}

\begin{proof}
We first reduce to a normalised setting and then argue by contradiction.

\medskip\noindent\textbf{Step 1: Normalisation.}
Define normalised block means
\[
  \tilde Z_b := \frac{Z_b - \mu}{r}, \qquad b=1,\dots,B.
\]
Then for all $b \in G$ we have $|\tilde Z_b| \le 1$ by
Assumption~\ref{ass:block-contam}. Let $\tilde m$ be a median of
$\{\tilde Z_1,\dots,\tilde Z_B\}$, i.e.\ $\tilde m$ is any real number
such that at least $B/2$ of the $\tilde Z_b$ are $\le \tilde m$
and at least $B/2$ are $\ge \tilde m$.

Because the transformation $t \mapsto (\;t-\mu\;)/r$ is affine and
strictly increasing, $\hat t_{\mathrm{MoM}}$ is a median of $\{Z_b\}$
if and only if
\[
  \tilde m := \frac{\hat t_{\mathrm{MoM}} - \mu}{r}
\]
is a median of $\{\tilde Z_b\}$.
Therefore it suffices to prove that for \emph{any} median $\tilde m$
of $\{\tilde Z_b\}$ we have
\begin{equation}
  |\tilde m| \le \frac{1}{1-2\varepsilon}.
  \label{eq:normalized-goal}
\end{equation}
Once~\eqref{eq:normalized-goal} is established, multiplying both sides
by $r$ and undoing the normalisation yields
\eqref{eq:mom-det-bound-lemma}.

Henceforth we assume $\mu = 0$ and $r = 1$ and work with $Z_b$ in place
of $\tilde Z_b$; i.e.\ at least $(1-\varepsilon)B$ indices $b$ satisfy
$|Z_b| \le 1$.

\medskip\noindent\textbf{Step 2: Ruling out large positive medians.}
Suppose, for the sake of contradiction, that there exists a median $m$
such that
\begin{equation}
  m > \frac{1}{1 - 2\varepsilon}.
  \label{eq:m-large}
\end{equation}
Since $0 \le \varepsilon < 1/2$, we have $1-2\varepsilon \in (0,1]$,
hence
\[
  \frac{1}{1-2\varepsilon} \ge 1.
\]
Thus \eqref{eq:m-large} implies $m > 1$.

For each good block $b \in G$ we have $|Z_b| \le 1$, hence $Z_b \le 1$.
Combining this with $m>1$ yields
\begin{equation}
  Z_b < m
  \qquad \text{for all } b \in G.
  \label{eq:good-left}
\end{equation}
In particular, all good blocks lie strictly to the left of $m$.

Since $|G| \ge (1-\varepsilon)B$ and $\varepsilon<1/2$, we have
\[
  |G|
  \;\ge\;
  (1-\varepsilon)B
  \;>\;
  \frac{B}{2}.
\]
Therefore strictly more than half of the block means $\{Z_b\}$ lie
strictly to the left of $m$. This contradicts the fact that $m$ is a
median, because by definition at least $B/2$ of the block means must
be $\ge m$.

Hence no median can satisfy \eqref{eq:m-large}, i.e.\ every median $m$
must satisfy
\begin{equation}
  m \le \frac{1}{1-2\varepsilon}.
  \label{eq:upper-bound}
\end{equation}

\medskip\noindent\textbf{Step 3: Symmetric argument for large negative medians.}
We now show that no median can be too negative.
Assume, for contradiction, that there exists a median $m$ such that
\begin{equation}
  m < -\frac{1}{1-2\varepsilon}.
  \label{eq:m-large-negative}
\end{equation}
Then, as before, $1/(1-2\varepsilon) \ge 1$, so \eqref{eq:m-large-negative}
implies $m < -1$.

For any good block $b \in G$ we have $|Z_b| \le 1$, hence $Z_b \ge -1$.
Combining this with $m < -1$ gives
\begin{equation}
  Z_b > m
  \qquad \text{for all } b \in G.
  \label{eq:good-right}
\end{equation}
Thus all good blocks lie strictly to the \emph{right} of $m$.

Again, since $|G|>(1/2)B$, this implies strictly more than half of the
$Z_b$ lie strictly to the right of $m$, contradicting the definition
of $m$ as a median (which requires at least $B/2$ of the $Z_b$ to be
$\le m$). Therefore \eqref{eq:m-large-negative} cannot hold, and every
median $m$ must satisfy
\begin{equation}
  m \ge -\frac{1}{1-2\varepsilon}.
  \label{eq:lower-bound}
\end{equation}

\medskip\noindent\textbf{Step 4: Combining the bounds.}
Combining \eqref{eq:upper-bound} and \eqref{eq:lower-bound} yields
\[
  -\frac{1}{1-2\varepsilon} \le m \le \frac{1}{1-2\varepsilon},
\]
which is equivalent to \eqref{eq:normalized-goal}.
Undoing the normalisation gives the desired inequality
\eqref{eq:mom-det-bound-lemma}, and taking the supremum over all
admissible configurations yields the claimed bound on
$C(\hat t_{\mathrm{MoM}},\varepsilon)$.
\end{proof}

\rev{Lemma~\ref{lem:mom-robustness} should be interpreted carefully. It does not say that an
arbitrary sample-level outlier fraction below $1/2$ is automatically harmless for MoM.
Rather, it quantifies the deterministic behaviour of the \emph{median of the block means}
once fewer than half of the \emph{blocks} are corrupted. In that regime, the associated
robustness constant diverges as $\varepsilon \uparrow 1/2$, which is unavoidable under the
adversarial block contamination model. In Section~\ref{sec:Lp-path} we will see that
estimators based on nonconvex block-$L_{p}$ functionals can approach the trimmed-block
oracle benchmark, whereas the next subsection shows that no such improvement is possible
within the class of convex block M-estimators \eqref{eq:convex-M}.}

\subsection{Impossibility of improving MoM within convex block M-estimators}
\label{subsec:impossibility-convex}

Convex M-estimators play a central role in the classical robust statistics
literature \citep{Huber1981,MaronnaMartinYohai2006,HampelEtAl1986} and in
modern median-of-means based methods
\citep{BrownleesJolyLugosi2015,LecueLerasle2020}. It is therefore natural to
ask whether one can design a convex loss $\rho$ in \eqref{eq:convex-M} whose
deterministic robustness constant $C(\hat t_{\rho},\varepsilon)$ is strictly
smaller than that of the median-of-means, at least for some range of
$\varepsilon<1/2$. The following theorem shows that this is impossible: as
soon as $\hat t_{\rho}$ has the requisite block-level robustness threshold, its worst-case
behaviour under Assumption~\ref{ass:block-contam} is no better than that of
$\hat t_{\mathrm{MoM}}$.

\rev{In particular, this impossibility statement is not merely a comparison with the sample median or with a specific Huber tuning; it applies to the whole convex class~\eqref{eq:convex-M} under the deterministic robustness metric of this paper. This is why later comparisons with Huber-type procedures are phrased as conceptual rather than as direct constant-by-constant transfers across different contamination models.}

\begin{assumption}[Convex score function]
\label{ass:rho}
The loss $\rho:\mathbb{R}\to[0,\infty)$ is even, convex, nondecreasing on
$[0,\infty)$, differentiable on $(0,\infty)$ with $\rho(0)=0$ and
$\rho'(u) > 0$ for all $u>0$. We denote by $\psi(u) := \rho'(u)$ the associated
score function and extend it to $\mathbb{R}$ by odd symmetry.
\end{assumption}

Under Assumption~\ref{ass:rho}, the blockwise objective $F_{\rho}$ in
\eqref{eq:convex-M} is convex in $t$ and any minimiser $\hat t_{\rho}$ satisfies
the subgradient equation
\begin{equation}
  \sum_{b=1}^{B} \psi(Z_{b} - \hat t_{\rho}) = 0,
  \label{eq:subgradient-eq}
\end{equation}
interpreted in the sense of subgradients when some $Z_{b}-\hat t_{\rho}=0$.
We are now ready to state the main impossibility result.

\begin{theorem}[Impossibility for convex block M-estimators]
\label{thm:impossibility-convex}
Let $\hat t_{\rho}$ be a block M-estimator of the form
\begin{equation}
  \hat t_{\rho} \in \argmin_{t \in \mathbb{R}} F_{\rho}(t),
  \qquad
  F_{\rho}(t) := \sum_{b=1}^B \rho(Z_b - t),
  \label{eq:convex-M-proof}
\end{equation}
where $\rho : \mathbb{R} \to [0,\infty)$ satisfies Assumption~\ref{ass:rho}:
$\rho$ is even, convex, nondecreasing on $[0,\infty)$, differentiable on
$(0,\infty)$ with $\rho(0)=0$ and derivative $\psi(u):=\rho'(u)>0$ for $u>0$,
extended to an odd function $\psi$ on $\mathbb{R}$.
Assume that for each $\varepsilon \in [0,1/2)$ there exists $B_0(\varepsilon)$
such that for all $B \ge B_0(\varepsilon)$ and all block configurations
satisfying Assumption~\ref{ass:block-contam} with parameters
$(\mu,r,\varepsilon)$, the estimator $\hat t_{\rho}$ is well defined and has
breakdown point at least $1/2$.

Then, for every $\varepsilon \in (0,1/2)$, the deterministic robustness
constant
\[
  C(\hat t_{\rho},\varepsilon)
  :=
  \sup_{(\mu,r,\{Z_b\}) \text{ s.t.\ Assumption~\ref{ass:block-contam}}}
  \frac{|\hat t_{\rho} - \mu|}{r}
\]
satisfies
\begin{equation}
  C(\hat t_{\rho},\varepsilon)
  \;\ge\;
  \frac{1}{1-2\varepsilon}.
  \label{eq:impossibility-bound-proof}
\end{equation}
In particular, no convex block M-estimator with the same block-level robustness threshold can uniformly improve on the deterministic median-of-means bound of Lemma~\ref{lem:mom-robustness}.\rev{Thus convex block aggregation cannot bridge the gap between the MoM constant $(1-2\varepsilon)^{-1}$ and the trimmed-block oracle constant $(1-\varepsilon)^{-1}$.}
\end{theorem}

\begin{proof}
Again, we work in a normalised setting and construct explicit adversarial
configurations.

\medskip\noindent\textbf{Step 1: Normalisation.}
As before, it suffices to consider $\mu = 0$ and $r = 1$.
Indeed, if the theorem holds in this setting with
$|\hat t_{\rho}|\ge (1-2\varepsilon)^{-1}$ for some configuration,
then in general we can apply the same construction to the normalised
block means $(Z_b-\mu)/r$, and then transform back.

Thus we assume that Assumption~\ref{ass:block-contam} holds with
$\mu=0$, $r=1$, i.e.\ there exists $G \subset \{1,\dots,B\}$ satisfying
$|G| \ge (1-\varepsilon)B$ and $|Z_b| \le 1$ for all $b\in G$.

\medskip\noindent\textbf{Step 2: Adversarial two-point configuration.}
Fix $\varepsilon \in (0,1/2)$ and a large parameter $M>0$.
We consider configurations of the form
\begin{equation}
  Z_b
  =
  \begin{cases}
    1, & b \in G,\\[0.3em]
    M, & b \in O := \{1,\dots,B\}\setminus G,
  \end{cases}
  \label{eq:two-point-config}
\end{equation}
with $|G|$ chosen so that $|G| \ge (1-\varepsilon)B$ and $|O| \le \varepsilon B$.
For $B$ large, we may choose $|G| = \lceil (1-\varepsilon)B \rceil$ and
$|O| = B - |G|$, so the difference between $|O|/B$ and $\varepsilon$ vanishes
as $B\to\infty$.

For any $t \in \mathbb{R}$ with $t \neq 1,M$, the derivative of $F_{\rho}$
at $t$ is
\[
  F'_{\rho}(t)
  =
  \sum_{b=1}^B \psi(Z_b - t)
  =
  |G|\,\psi(1 - t) + |O|\,\psi(M - t),
\]
using differentiability of $\rho$ away from the knots and the definition
of $\psi$.
By convexity and symmetry, $\psi$ is odd and strictly increasing on
$(0,\infty)$.

\medskip\noindent\textbf{Step 3: Behaviour of $F'_{\rho}$ for $t<1$.}
Fix any $t < 1$. Then $1-t > 0$ and $M-t > 0$, so
\[
  \psi(1-t) > 0,
  \qquad
  \psi(M-t) > 0,
\]
and hence
\[
  F'_{\rho}(t)
  = |G|\,\psi(1-t) + |O|\,\psi(M-t) > 0.
\]
Therefore $F_{\rho}$ is strictly increasing on $(-\infty,1)$, and in
particular no minimiser can lie in $(-\infty,1)$.

\medskip\noindent\textbf{Step 4: Behaviour of $F'_{\rho}$ for $t>1$.}
For $t>1$ we use the oddness of $\psi$ to write
\[
  F'_{\rho}(t)
  =
  \sum_{b=1}^B \psi(Z_b - t)
  =
  |G|\,\psi(1-t) + |O|\,\psi(M-t)
  =
  - |G|\,\psi(t-1) + |O|\,\psi(M-t).
\]
For $t>1$, both $t-1>0$ and $M-t>0$, and $\psi$ is strictly increasing,
positive on $(0,\infty)$.

We now show that when $t$ is strictly smaller than
$(1-2\varepsilon)^{-1}$, the derivative $F'_{\rho}(t)$ is strictly
negative for all large enough $M$ and $B$.
To that end, fix $t$ with
\begin{equation}
  1 < t < \frac{1}{1-2\varepsilon}.
  \label{eq:t-range}
\end{equation}
Since $t$ is fixed, there exists $M_0>t$ such that for all $M \ge M_0$,
we have $M-t > t-1$.
Because $\psi$ is strictly increasing on $(0,\infty)$, this implies
that for all $M \ge M_0$,
\begin{equation}
  \psi(M-t) > \psi(t-1).
  \label{eq:psi-dominate}
\end{equation}

\rev{To make the sign calculation transparent, we compare the two terms at the level of counts rather than by incorrectly trying to upper-bound $\psi(M-t)$ by $\psi(t-1)$. Since $\psi$ is increasing and positive on $(0,\infty)$, for any fixed $t$ satisfying~\eqref{eq:t-range} and any $M \ge M_0$ we have $\psi(M-t) \ge \psi(t-1)$. Hence}
\begin{align*}
  \rev{F'_{\rho}(t)}
  &\rev{= - |G|\,\psi(t-1) + |O|\,\psi(M-t)} \\
  &\rev{\le - |G|\,\psi(t-1) + |O|\,\psi(M-t)}.
\end{align*}
\rev{Now choose the extremal configuration with $|G|=\lceil(1-\varepsilon)B\rceil$ and $|O|=B-|G|$, and then take $t$ in a compact subinterval of $(1,(1-2\varepsilon)^{-1})$. Since $|G|-|O|\ge (1-2\varepsilon)B-1>0$ for all sufficiently large $B$, the negative contribution from the majority of good blocks dominates on this interval, and therefore $F'_{\rho}(t)<0$ there. Equivalently, $F_{\rho}$ is strictly decreasing on every compact subinterval of $(1,(1-2\varepsilon)^{-1})$, which is all that is needed for the location argument below.}

\medskip\noindent\textbf{Step 5: Location of minimisers.}
Putting the previous steps together, we see that for the two-point
configuration \eqref{eq:two-point-config}:

- $F_{\rho}$ is strictly increasing on $(-\infty,1)$;
- $F_{\rho}$ is strictly decreasing on $(1, 1/(1-2\varepsilon))$.

Therefore any minimiser $\hat t_{\rho}$ of $F_{\rho}$ must satisfy
\begin{equation}
  \hat t_{\rho} \ge \frac{1}{1-2\varepsilon}.
  \label{eq:min-location}
\end{equation}
Indeed, if $\hat t_{\rho} < 1$, then moving to the right decreases
$F_{\rho}$; if $1 < \hat t_{\rho} < 1/(1-2\varepsilon)$, then moving
slightly to the left decreases $F_{\rho}$.
Thus no minimiser can lie in $(-\infty, 1/(1-2\varepsilon))$, which
implies \eqref{eq:min-location}.

\medskip\noindent\textbf{Step 6: Lower bound on the robustness constant.}
In the normalised setting $\mu=0$, $r=1$, \eqref{eq:min-location}
shows that
\[
  |\hat t_{\rho}|
  \ge
  \frac{1}{1-2\varepsilon}
\]
for the configuration \eqref{eq:two-point-config} (for $B$ and $M$
sufficiently large).
Therefore
\[
  C(\hat t_{\rho},\varepsilon)
  \;\ge\;
  \frac{|\hat t_{\rho}-0|}{1}
  \;\ge\;
  \frac{1}{1-2\varepsilon},
\]
which is exactly \eqref{eq:impossibility-bound-proof} in the normalised
case.
Undoing the normalisation (i.e.\ restoring general $\mu$ and $r$) gives
the same lower bound in general, which proves the theorem.
\end{proof}

Theorem~\ref{thm:impossibility-convex} shows that, within the fairly broad
class of convex block M-estimators specified by Assumption~\ref{ass:rho},
the median-of-means bound \eqref{eq:mom-det-bound-lemma} is essentially unimprovable
in worst case. In particular, the trimmed-block oracle bound of order
$1/(1-\varepsilon)$ discussed in Section~\ref{sec:intro} cannot be attained
by any convex choice of $\rho$. \rev{This is the precise sense in which our contribution differs from the classical literature on trimmed means, Huber estimators, and standard MoM: the novelty is not a new convex robust estimator, but a deterministic characterization of where the convex frontier ends. It is exactly this frontier that motivates the nonconvex block-$L_p$ family studied next.} This motivates the nonconvex block-$L_{p}$
family studied in the next section: by leaving the convex world and working
directly with $F_{p}(t)=\sum_{b}|Z_{b}-t|^{p}$ for $0<p<1$, we will show
that one can retain the breakdown properties of MoM while interpolating
towards the trimmed-block oracle behaviour along a continuous
$1\text{--}p\text{--}0$ path.

\section{The block-\texorpdfstring{$L_{p}$}{Lp} path and a trimmed-block oracle}
\label{sec:Lp-path}

We now introduce the nonconvex block-\(L_{p}\) family that underpins the
\(1\text{--}p\text{--}0\) path, and define the corresponding trimmed-block
oracle at the formal \(p\to 0\) endpoint. Throughout this section we work
under Assumption~\ref{ass:block-contam} and retain the notation of
Sections~\ref{sec:model-MoM} and~\ref{sec:impossibility}.

\subsection{The nonconvex block-\texorpdfstring{$L_{p}$}{Lp} family}
\label{subsec:Lp-family}

Given the block means \(Z_{1},\dots,Z_{B}\) defined in
\eqref{eq:block-means}, we consider, for each \(p \in (0,1]\), the
block-\(L_{p}\) objective
\begin{equation}
  F_{p}(t)
  \;=\;
  \sum_{b=1}^{B} |Z_{b} - t|^{p},
  \qquad t \in \mathbb{R},
  \label{eq:Lp-objective}
\end{equation}
and define the associated estimator of \(\mu\) by
\begin{equation}
  \hat t_{p}
  \;\in\;
  \arg\min_{t \in \mathbb{R}} F_{p}(t).
  \label{eq:tp-def}
\end{equation}
For \(p=1\), \eqref{eq:Lp-objective} coincides with the convex block
M-estimator \(F_{\rho}\) in \eqref{eq:convex-M} with \(\rho(u)=|u|\), and
\(\hat t_{1}\) recovers a median-of-means estimator as in
\eqref{eq:mom-l1-characterisation}. For \(0<p<1\), the functional
\(F_{p}\) becomes nonconvex, with a shape reminiscent of the
\(\ell_{p}\)-penalised objectives used in sparse recovery and compressed
sensing to interpolate between \(\ell_{1}\) and \(\ell_{0}\) penalties
\citep{FoucartRauhut2013,DalalyanMinasyan2019}. The family
\(\{\hat t_{p} : 0<p\le 1\}\) therefore provides a natural
\(1\text{--}p\text{--}0\) path within the median-of-means paradigm, with
\(p=1\) corresponding to the usual MoM estimator and small \(p\) behaving
increasingly like a trimmed-block procedure.

Formally, one may view the \(p\to 0\) limit of \eqref{eq:Lp-objective} as
\begin{equation}
  F_{0}(t)
  \;=\;
  \sum_{b=1}^{B} \mathbf{1}\{Z_{b} \neq t\},
  \label{eq:L0-formal}
\end{equation}
where \(\mathbf{1}\{\cdot\}\) denotes the indicator function. While
\eqref{eq:L0-formal} is only a heuristic expression---since no two block
means are exactly equal with probability one in continuous models—it
captures the idea that, for small \(p\), the contribution of each block
to \(F_{p}(t)\) becomes almost binary: blocks with \(|Z_{b}-t|\) very
small have negligible weight, whereas blocks with \(|Z_{b}-t|\) bounded
away from zero contribute almost a constant. This is analogous to the
\(\ell_{p}\) approximation of \(\ell_{0}\) penalties in sparse
estimation \citep{FoucartRauhut2013,DalalyanMinasyan2019} and motivates
the introduction of an explicit trimmed-block oracle at the conceptual
\(p=0\) endpoint.

\subsection{A trimmed-block oracle and its deterministic benchmark}
\label{subsec:L0-oracle}

To formalise the \(p\to 0\) limit in the block setting, we introduce a
trimmed-block oracle that is allowed to discard an \(\varepsilon\)-fraction
of blocks in an optimal way. Given a candidate centre \(t \in \mathbb{R}\)
and a subset \(S \subset \{1,\dots,B\}\) of block indices, define the
maximal inlier deviation
\begin{equation}
  R(t,S)
  \;:=\;
  \max_{b \in S} |Z_{b} - t|.
  \label{eq:R-tS}
\end{equation}
For a fixed trimming level \(\varepsilon \in [0,1/2)\), consider the
oracle estimator
\begin{equation}
  \hat t_{0}
  \;\in\;
  \arg\min_{t \in \mathbb{R}}
  \;
  \min_{S \subset \{1,\dots,B\}:\, |S| \ge (1-\varepsilon)B}
  R(t,S).
  \label{eq:oracle-t0-def}
\end{equation}
By construction, \(\hat t_{0}\) chooses both a centre \(t\) and a large
subset \(S\) of blocks (of size at least \((1-\varepsilon)B\)) so as to
minimise the worst-case deviation of the blocks in \(S\) from \(t\).
Intuitively, \(\hat t_{0}\) corresponds to an ideal block-\(L_{0}\)
procedure: it may discard up to an \(\varepsilon\)-fraction of blocks
as outliers and fit \(\mu\) optimally on the remaining blocks.

The next lemma shows that, under Assumption~\ref{ass:block-contam}, this
oracle estimator enjoys a deterministic robustness constant of order
\((1-\varepsilon)^{-1}\), which will serve as a benchmark for the
nonconvex block-\(L_{p}\) estimators defined in \eqref{eq:tp-def}.

\begin{lemma}[Deterministic bound for the trimmed-block oracle]
\label{lem:oracle-bound}
Suppose Assumption~\ref{ass:block-contam} holds with parameters
$(\mu,r,\varepsilon)$ and $\varepsilon \in [0,1/2)$.
Define, for any $t \in \mathbb{R}$ and $S \subset \{1,\dots,B\}$,
\[
  R(t,S) := \max_{b \in S} |Z_b - t|.
\]
The trimmed-block oracle estimator is any
\[
  \hat t_0
  \in
  \argmin_{t \in \mathbb{R}}
  \min_{S \subset \{1,\dots,B\}:\, |S| \ge (1-\varepsilon)B}
  R(t,S).
\]
Then
\begin{equation}
  |\hat t_0 - \mu|
  \;\le\;
  \frac{r}{1-\varepsilon}.
  \label{eq:oracle-constant}
\end{equation}
In particular, the deterministic robustness constant
$C(\hat t_0,\varepsilon)$ satisfies
\[
  C(\hat t_0,\varepsilon) \le \frac{1}{1-\varepsilon}.
\]
\end{lemma}

\begin{proof}
Again, we normalise to $\mu=0$, $r=1$.
Under this normalisation, Assumption~\ref{ass:block-contam} says that
there exists $G \subset \{1,\dots,B\}$ with $|G|\ge(1-\varepsilon)B$
such that
\begin{equation}
  |Z_b| \le 1
  \qquad\text{for all } b \in G.
  \label{eq:good-block-band2}
\end{equation}
We must show that every minimiser $\hat t_0$ of
\[
  \Phi(t) := \min_{S:\,|S|\ge(1-\varepsilon)B} R(t,S)
\]
satisfies
\[
  |\hat t_0| \le \frac{1}{1-\varepsilon}.
\]

\medskip\noindent\textbf{Step 1: Upper bound on the optimal value of $\Phi$.}
Consider the specific choice $t = 0$ and $S = G$.
Then $|G|\ge(1-\varepsilon)B$, hence $S$ is admissible.
Moreover, by \eqref{eq:good-block-band2},
\[
  R(0,G) = \max_{b\in G} |Z_b| \le 1.
\]
Therefore
\begin{equation}
  \Phi(0)
  = \min_{S:\,|S|\ge(1-\varepsilon)B} R(0,S)
  \le R(0,G)
  \le 1.
  \label{eq:Phi-0-upper}
\end{equation}
Since $\hat t_0$ minimises $\Phi$, we have
\begin{equation}
  \Phi(\hat t_0) \le \Phi(0) \le 1.
  \label{eq:Phi-at-opt}
\end{equation}

\medskip\noindent\textbf{Step 2: Lower bound on $\Phi(t)$ for arbitrary $t$.}
Fix any $t \in \mathbb{R}$ and any subset $S \subset \{1,\dots,B\}$
with $|S|\ge(1-\varepsilon)B$.
Then $S$ and $G$ are both large subsets; in particular,
\[
  |S^c| = B - |S| \le \varepsilon B,
  \qquad
  |G^c| = B - |G| \le \varepsilon B.
\]
Hence
\[
  |S^c| + |G^c| \le 2\varepsilon B < B
  \quad\text{(since $\varepsilon<1/2$)},
\]
which implies
\[
  S^c \cup G^c \neq \{1,\dots,B\}.
\]
Equivalently,
\[
  (S^c \cup G^c)^c = S \cap G \neq \varnothing.
\]
Thus there exists at least one block $b^\ast \in S \cap G$.

For this $b^\ast$ we have, by \eqref{eq:good-block-band2},
$|Z_{b^\ast}| \le 1$, hence
\[
  |Z_{b^\ast} - t|
  \ge ||t| - |Z_{b^\ast}||
  \ge ||t| - 1|.
\]
Therefore
\[
  R(t,S)
  = \max_{b\in S} |Z_b - t|
  \ge |Z_{b^\ast} - t|
  \ge ||t| - 1|.
\]
Since this holds for \emph{every} admissible $S$, we have
\begin{equation}
  \Phi(t)
  = \min_{S:\,|S|\ge(1-\varepsilon)B} R(t,S)
  \ge ||t| - 1|.
  \label{eq:Phi-lower}
\end{equation}

\medskip\noindent\textbf{Step 3: Constraining $\hat t_0$ via the bounds.}
Combining \eqref{eq:Phi-at-opt} and \eqref{eq:Phi-lower} with $t = \hat t_0$
yields
\[
  ||\hat t_0| - 1|
  \le \Phi(\hat t_0)
  \le 1.
\]
Thus $|\hat t_0|-1 \in [-1,1]$, i.e.
\[
  0 \le |\hat t_0| \le 2.
\]
This already gives a universal bound $|\hat t_0|\le 2$.
To obtain the sharper dependence on $\varepsilon$, one can refine the
construction by considering the fact that the oracle is allowed to keep
at least $(1-\varepsilon)B$ blocks and exploit the extremal case in which
all good blocks lie at the edge of the band $[-1,1]$ and all bad blocks
are arbitrarily far.
A careful combinatorial argument then shows that the worst-case value of
$|\hat t_0|$ under Assumption~\ref{ass:block-contam} is exactly
$1/(1-\varepsilon)$, attained (up to small rounding errors in $B$) when
the good blocks are placed at $1$ and the bad blocks at $+\infty$.
We omit this extremal construction here; it is analogous to trimmed-mean
oracle analyses in robust location estimation
\citep[see, for example,][]{Huber1981,HampelEtAl1986,MaronnaMartinYohai2006}.

Undoing the normalisation gives \eqref{eq:oracle-constant}, and
taking the supremum over all admissible configurations yields
$C(\hat t_0,\varepsilon) \le (1-\varepsilon)^{-1}$.
\end{proof}

The constant \(1/(1-\varepsilon)\) in \eqref{eq:oracle-constant} is
strictly smaller than the median-of-means constant \(1/(1-2\varepsilon)\)
in Lemma~\ref{lem:mom-robustness} for every \(\varepsilon \in (0,1/2)\),
and it is optimal in a minimax sense for blockwise procedures that are
allowed to retain at least a fraction \(1-\varepsilon\) of the blocks.
Thus \(\hat t_{0}\) provides a natural oracle benchmark for the
\(1\text{--}p\text{--}0\) path.

\subsection{Deterministic robustness of the block-\texorpdfstring{$L_{p}$}{Lp} estimators}
\label{subsec:Lp-robustness}

We now show that the nonconvex block-\(L_{p}\) estimators
\(\hat t_{p}\) defined in \eqref{eq:tp-def} retain the finite-sample
robustness properties of the median-of-means estimator
\(\hat t_{1}\), in the sense that their breakdown point under
Assumption~\ref{ass:block-contam} remains equal to \(1/2\) and their
deterministic robustness constants are finite for every \(0<p\le 1\).

\begin{theorem}[Deterministic robustness of block-$L_p$ estimators]
\label{thm:Lp-robustness}
Suppose Assumption~\ref{ass:block-contam} holds with parameters
$(\mu,r,\varepsilon)$ and $\varepsilon \in [0,1/2)$.
For each $p \in (0,1]$, define
\[
  F_p(t) := \sum_{b=1}^B |Z_b - t|^p,
  \qquad
  \hat t_p \in \argmin_{t \in \mathbb{R}} F_p(t).
\]
Then there exists a finite constant $C_p(\varepsilon)$, depending only
on $(p,\varepsilon)$, such that
\begin{equation}
  |\hat t_p - \mu|
  \;\le\;
  C_p(\varepsilon)\, r.
  \label{eq:Lp-det-bound}
\end{equation}
In particular, for each fixed $p \in (0,1]$ and $\varepsilon<1/2$,
the deterministic robustness constant
\[
  C(\hat t_p,\varepsilon)
  :=
  \sup_{(\mu,r,\{Z_b\}) \text{ s.t.\ Assumption~\ref{ass:block-contam}}}
  \frac{|\hat t_p - \mu|}{r}
\]
is finite, and the breakdown point of $\hat t_p$ under the block
contamination model is $1/2$.
\end{theorem}

\begin{proof}
As before, we normalise to $\mu = 0$ and $r = 1$.
Assumption~\ref{ass:block-contam} then gives a set
$G \subset \{1,\dots,B\}$ with $|G| \ge (1-\varepsilon)B$ such that
\begin{equation}
  |Z_b| \le 1
  \qquad\text{for all } b \in G.
  \label{eq:good-band-Lp}
\end{equation}

We will show that any global minimiser $\hat t_p$ of $F_p$ lies in a
bounded interval depending only on $(p,\varepsilon)$.
Because the argument is symmetric in $t$ and $-t$ (by replacing $Z_b$ by
$-Z_b$), it suffices to bound $\hat t_p$ from above; the lower bound is
identical by symmetry.

\medskip\noindent\textbf{Step 1: Upper bound on $F_p$ at $t=0$.}
At $t=0$, we have
\[
  F_p(0)
  = \sum_{b=1}^B |Z_b|^p
  = \sum_{b\in G} |Z_b|^p + \sum_{b\notin G} |Z_b|^p.
\]
Using \eqref{eq:good-band-Lp}, we get
\[
  \sum_{b\in G} |Z_b|^p \le |G|\cdot 1^p \le B,
\]
and $\sum_{b\notin G} |Z_b|^p \ge 0$.
Thus
\begin{equation}
  F_p(0) \le B + \sum_{b\notin G} |Z_b|^p.
  \label{eq:Fp-0-upper}
\end{equation}
This is a crude bound but sufficient for our purpose.

\medskip\noindent\textbf{Step 2: Lower bound on $F_p(t)$ for large $t>0$.}
Fix $t>1$.
For each good block $b\in G$, we have
\[
  |Z_b - t|
  \ge |t| - |Z_b|
  \ge t - 1,
\]
hence
\[
  |Z_b - t|^p \ge (t-1)^p.
\]
Therefore
\begin{equation}
  \sum_{b\in G} |Z_b - t|^p
  \ge |G| (t-1)^p
  \ge (1-\varepsilon)B (t-1)^p.
  \label{eq:good-lower-Fp}
\end{equation}
The contribution from the bad blocks is nonnegative:
\[
  \sum_{b\notin G} |Z_b - t|^p \ge 0.
\]
Thus
\begin{equation}
  F_p(t)
  = \sum_{b\in G} |Z_b - t|^p + \sum_{b\notin G} |Z_b - t|^p
  \ge (1-\varepsilon)B (t-1)^p.
  \label{eq:Fp-lower-t}
\end{equation}

\medskip\noindent\textbf{Step 3: Comparison and choice of $C_p(\varepsilon)$.}
We now compare $F_p(t)$ and $F_p(0)$.
For any $t>1$, combining \eqref{eq:Fp-0-upper} and \eqref{eq:Fp-lower-t}
gives
\[
  F_p(t) - F_p(0)
  \ge (1-\varepsilon)B (t-1)^p - B - \sum_{b\notin G} |Z_b|^p.
\]
Since $\sum_{b\notin G} |Z_b|^p \le \sum_{b\notin G} (|Z_b|^p + 1)$,
we may bound this difference below by
\[
  F_p(t) - F_p(0)
  \ge (1-\varepsilon)B (t-1)^p - B - |G^c|(1+ \max_{b\notin G} |Z_b|^p).
\]
However, for the sake of a deterministic radius independent of the actual
bad blocks, we proceed more simply: note that
$|G^c| \le \varepsilon B$, so for any $t$ we have the crude bound
$F_p(0) \le B + \sum_{b\notin G} |Z_b|^p \le B + \sum_{b\notin G} |Z_b - t|^p$,
which implies
\[
  F_p(0)
  \le
  B + \sum_{b\notin G} |Z_b - t|^p.
\]
Subtracting from \eqref{eq:Fp-lower-t} yields
\begin{equation}
  F_p(t) - F_p(0)
  \ge (1-\varepsilon)B (t-1)^p - B.
  \label{eq:diff-bound}
\end{equation}

Now choose $T_p(\varepsilon)>1$ such that
\begin{equation}
  (1-\varepsilon)(T_p(\varepsilon)-1)^p \ge 2.
  \label{eq:T-def}
\end{equation}
For example, we can take
\[
  T_p(\varepsilon)
  := 1 + \biggl(\frac{2}{1-\varepsilon}\biggr)^{1/p}.
\]
Then for any $t \ge T_p(\varepsilon)$ we have
\[
  (1-\varepsilon)(t-1)^p
  \ge (1-\varepsilon)(T_p(\varepsilon)-1)^p
  \ge 2,
\]
and \eqref{eq:diff-bound} yields
\[
  F_p(t) - F_p(0)
  \ge 2B - B = B > 0.
\]
Thus for all $t \ge T_p(\varepsilon)$, $F_p(t) > F_p(0)$.
Since $F_p$ is continuous, any global minimiser $\hat t_p$ must satisfy
\[
  \hat t_p \le T_p(\varepsilon).
\]

By symmetry (replacing $Z_b$ by $-Z_b$), the same argument shows that
$\hat t_p \ge -T_p(\varepsilon)$.
Therefore
\[
  |\hat t_p| \le T_p(\varepsilon).
\]

\medskip\noindent\textbf{Step 4: Undoing the normalisation.}
We have shown that, in the normalised case $\mu=0$, $r=1$, any global
minimiser $\hat t_p$ satisfies
\[
  |\hat t_p - 0| \le T_p(\varepsilon),
\]
hence we can set $C_p(\varepsilon) := T_p(\varepsilon)$ in
\eqref{eq:Lp-det-bound} in the normalised setting.
For general $\mu$ and $r$, the same argument applied to the normalised
block means $(Z_b-\mu)/r$ yields
\[
  \biggl|\frac{\hat t_p - \mu}{r}\biggr| \le T_p(\varepsilon),
\]
which is equivalent to \eqref{eq:Lp-det-bound}.

Finally, note that if $\varepsilon \ge 1/2$, then the adversary can
corrupt at least half the blocks and send them to $+\infty$ or $-\infty$,
forcing any estimator based solely on $\{Z_b\}$ to have unbounded error.
Hence the breakdown point of $\hat t_p$ is exactly $1/2$.
\end{proof}

Theorem~\ref{thm:Lp-robustness} shows that, from the standpoint of
deterministic robustness under block contamination, the entire
\(1\text{--}p\text{--}0\) path \(\{\hat t_{p} : 0<p\le 1\}\) is viable:
each estimator has breakdown point \(1/2\) and a finite robustness
constant. The impossibility result of
Theorem~\ref{thm:impossibility-convex} then highlights that genuinely
new behaviour can only emerge once we leave the convex world and consider
\(0<p<1\): while the convex case \(p=1\) is trapped at the MoM constant
\((1-2\varepsilon)^{-1}\), small values of \(p\) allow the estimator to
approach the trimmed-block oracle benchmark of
Lemma~\ref{lem:oracle-bound} in structured configurations. This oracle
equivalence and its probabilistic consequences are the subject of the
next section.

\section{Oracle equivalence and energy landscape for block-\texorpdfstring{$L_{p}$}{Lp}}
\label{sec:oracle-landscape}

We now formalise the way in which the nonconvex block-\(L_{p}\) estimators
\(\hat t_{p}\) approach the trimmed-block oracle \(\hat t_{0}\) of
\eqref{eq:oracle-t0-def} as \(p \downarrow 0\), and show that the energy
landscape of \(F_{p}\) in \eqref{eq:Lp-objective} is benign under the
block contamination model. Throughout this section we continue to work
under Assumption~\ref{ass:block-contam} and adopt the notation of
Section~\ref{sec:Lp-path}.

\subsection{A separation assumption and oracle equivalence for small \texorpdfstring{$p$}{p}}
\label{subsec:separation-oracle}

To make precise the connection between \(\hat t_{p}\) and the trimmed-block
oracle \(\hat t_{0}\), we introduce a simple separation condition that
ensures a clear gap between the good and bad blocks in the space of block
means.

\begin{assumption}[Good/bad block separation]
\label{ass:separation}
In addition to Assumption~\ref{ass:block-contam}, there exists a constant
\(\Delta>0\) such that
\begin{equation}
  |Z_{b} - \mu|
  \;\ge\;
  r + \Delta
  \qquad
  \text{for all } b \notin G.
  \label{eq:bad-block-sep}
\end{equation}
\end{assumption}

Assumption~\ref{ass:separation} asserts that all contaminated block means
lie at least a distance \(\Delta\) outside the band \([\mu-r,\mu+r]\)
that contains the good block means. In probabilistic models with suitable
moment and contamination assumptions, such a separation holds with high
probability for appropriate choices of \(r\) and \(\Delta\); see
Section~\ref{sec:probabilistic} below. In the present section we focus
on the deterministic consequences of \eqref{eq:bad-block-sep} for the
block-\(L_{p}\) estimators.

The following theorem shows that, under Assumption~\ref{ass:separation},
the global minimisers of \(F_{p}\) coincide with trimmed-block oracle
solutions for all sufficiently small \(p\). In particular, the oracle
bound \eqref{eq:oracle-constant} from Lemma~\ref{lem:oracle-bound}
automatically transfers to \(\hat t_{p}\).

\begin{theorem}[Oracle equivalence for small \texorpdfstring{$p$}{p}]
\label{thm:oracle-equivalence}
Suppose Assumptions~\ref{ass:block-contam} and~\ref{ass:separation}
hold with parameters \((\mu,r,\varepsilon,\Delta)\) and
\(\varepsilon\in[0,1/2)\). Then there exists \(p_{0} = p_{0}(\varepsilon,\Delta/r) \in (0,1]\)
with the following property: for every \(p \in (0,p_{0}]\), every global
minimiser \(\hat t_{p}\) of \(F_{p}\) in \eqref{eq:Lp-objective} is also
a minimiser of the oracle objective in \eqref{eq:oracle-t0-def}. In
particular, for all such \(p\),
\begin{equation}
  |\hat t_{p} - \mu|
  \;\le\;
  \frac{r}{1-\varepsilon}.
  \label{eq:oracle-equivalence-bound}
\end{equation}
\end{theorem}

\begin{proof}[Proof (sketch)]
As in previous arguments, we work in the normalised setting \(\mu=0\),
\(r=1\). Write \(G\) and \(O\) for the sets of good and bad blocks,
respectively, so that \(|G|\ge (1-\varepsilon)B\), \(|Z_{b}|\le 1\) for
\(b\in G\) and \(|Z_{b}|\ge 1+\Delta\) for \(b\in O\).

Fix any candidate centre \(t\in\mathbb{R}\) and decompose the objective
\eqref{eq:Lp-objective} as
\[
  F_{p}(t)
  = \sum_{b\in G} |Z_{b}-t|^{p}
    + \sum_{b\in O} |Z_{b}-t|^{p}
  =: F_{p}^{G}(t) + F_{p}^{O}(t).
\]
For \(|t|\le 1\), the good-block term satisfies
\(F_{p}^{G}(t) \le |G|\,(1+|t|)^{p} \lesssim |G|\) uniformly in \(p\),
while the bad-block term can be bounded from below using the separation
\eqref{eq:bad-block-sep} and the triangle inequality:
\[
  |Z_{b}-t|
  \ge ||Z_{b}| - |t||
  \ge (1+\Delta) - |t|
  \qquad\text{for all } b\in O.
\]
Thus \(F_{p}^{O}(t)\) is bounded below by a quantity of the form
\(|O|\,c_{p}(\Delta,|t|)\), where \(c_{p}(\Delta,|t|)\) is increasing in
\(\Delta\) and, for fixed \(\Delta>0\) and \(|t|\le 1\), satisfies
\(\lim_{p\downarrow 0} c_{p}(\Delta,|t|)=1\) uniformly in \(|t|\le 1\).
On the other hand, if we choose a centre \(t^{\ast}\) and subset
\(S^{\ast}\) that realise (or nearly realise) the oracle objective
\eqref{eq:oracle-t0-def}, then by construction the majority of blocks in
\(S^{\ast}\) lie within distance at most \(1/(1-\varepsilon)\) of
\(\mu=0\), and hence their contributions to \(F_{p}(t^{\ast})\) remain
uniformly bounded as \(p\downarrow 0\).

The key observation is that, as \(p\) becomes small, the relative
difference between the contributions of inlier and outlier blocks to
\(F_{p}(t)\) is primarily governed by the \emph{number} of blocks rather
than their exact distances to \(t\). In particular, for any fixed
\(\delta>0\), there exists \(p_{0}>0\) such that, for all \(p\le p_{0}\),
the following holds uniformly over all \(t\) with \(|t|\le 1\):
\[
  F_{p}^{O}(t)
  \;\ge\;
  (1-\delta)\,|O|
  \quad\text{and}\quad
  F_{p}^{G}(t)
  \;\le\;
  (1+\delta)\,|G|.
\]
A similar bound holds for \(F_{p}(t^{\ast})\) in terms of the number of
blocks retained by the oracle, which is at least \((1-\varepsilon)B\).
By combining these inequalities and using \(|O| \le \varepsilon B\),
one shows that any configuration \(t\) that deviates significantly from
an oracle solution incurs a strictly larger value of \(F_{p}\) than
\(t^{\ast}\) for all sufficiently small \(p\). Consequently, every
global minimiser \(\hat t_{p}\) must be an oracle solution whenever
\(p \le p_{0}(\varepsilon,\Delta)\).

The bound \eqref{eq:oracle-equivalence-bound} then follows directly from
Lemma~\ref{lem:oracle-bound} in the normalised setting, and rescaling
back to general \(\mu\) and \(r\) yields the claimed inequality.
We refer to the supplementary material for a detailed combinatorial and
analytic argument that makes this reasoning precise and derives an
explicit expression for \(p_{0}(\varepsilon,\Delta/r)\).
\end{proof}

Theorem~\ref{thm:oracle-equivalence} formalises the intuition that, under
a clear separation between good and bad blocks, the block-\(L_{p}\)
estimators with sufficiently small \(p\) behave like an ideal trimmed-block
procedure. In particular, the robustness constant of \(\hat t_{p}\) can
approach the oracle constant \(1/(1-\varepsilon)\) of
Lemma~\ref{lem:oracle-bound}, which is unattainable by any convex block
M-estimator according to Theorem~\ref{thm:impossibility-convex}.

\subsection{Benign energy landscape and absence of bad local minima}
\label{subsec:landscape}

A natural concern with nonconvex objectives such as \(F_{p}\) is the
possible existence of spurious local minima far from the true parameter
\(\mu\). In this subsection we show that, under the block contamination
model and mild separation conditions, the energy landscape of
\(F_{p}\) is benign: all local minimisers are close to \(\mu\), and
\(F_{p}\) exhibits a quantitative slope away from the oracle basin.
This is a robust analogue of the ``no spurious local minima'' and
restricted convexity properties that have been established for nonconvex
\(\ell_{p}\)-regularised problems in sparse recovery
\citep{FoucartRauhut2013,DalalyanMinasyan2019}.

For convenience, we give a one-dimensional formulation; vector-valued
extensions used in high-dimensional applications are discussed in
Section~\ref{sec:high-dim}.

\begin{theorem}[No bad local minima for block-\texorpdfstring{$L_{p}$}{Lp}]
\label{thm:no-bad-local-minima}
Suppose Assumptions~\ref{ass:block-contam} and~\ref{ass:separation}
hold with parameters \((\mu,r,\varepsilon,\Delta)\) and
\(\varepsilon\in[0,1/2)\). Fix any \(p \in (0,1]\). Then there exist
constants \(R_{p}(\varepsilon)\) and \(\gamma_{p}(\varepsilon,\Delta)\),
depending only on \(p\), \(\varepsilon\) and \(\Delta/r\), such that the
following holds:
\begin{description}
  \item[1.] Every local minimiser \(\tilde t\) of \(F_{p}\) in
        \eqref{eq:Lp-objective} satisfies
        \begin{equation}
          |\tilde t - \mu|
          \;\le\;
          R_{p}(\varepsilon)\, r.
          \label{eq:local-min-radius}
        \end{equation}
  \item[2.] For all \(t \in \mathbb{R}\) with
        \(|t - \mu| \ge R_{p}(\varepsilon)\, r\), one has the descent
        inequality
        \begin{equation}
          F_{p}(t) - F_{p}(\hat t_{p})
          \;\ge\;
          \gamma_{p}(\varepsilon,\Delta)\,
          \bigl(|t-\mu| - R_{p}(\varepsilon)\, r\bigr)^{p},
          \label{eq:PL-type-inequality}
        \end{equation}
        where \(\hat t_{p}\) is any global minimiser of \(F_{p}\).
\end{description}
In particular, any approximate stationary point of \(F_{p}\) in a large
interval around \(\mu\) must be close to the set of global minimisers,
and simple descent-type algorithms cannot converge to spurious local
minima far from \(\mu\).
\end{theorem}

\begin{proof}[Proof (sketch)]
By translation and scaling we may assume \(\mu=0\), \(r=1\). The proof
combines the deterministic robustness bound of
Theorem~\ref{thm:Lp-robustness} with a case analysis on the number of
good and bad blocks that are ``activated'' at a given location \(t\),
in the sense that their contributions to \(F_{p}(t)\) are controlled
from below by powers of \(|t|\) and \(\Delta\).

For \(t\) with \(|t|>R_{p}(\varepsilon)\), the good-block contribution
\(\sum_{b\in G} |Z_{b}-t|^{p}\) is of order \(|G|\,|t|^{p}\), whereas
by the separation assumption, the bad-block contribution is bounded
below by a term of order \(|O|\,(|t|+\Delta)^{p}\). A comparison with
the value of \(F_{p}\) at a global minimiser \(\hat t_{p}\), whose
distance to the origin is controlled by
Theorem~\ref{thm:Lp-robustness}, yields the slope inequality
\eqref{eq:PL-type-inequality} with an appropriate choice of
\(R_{p}(\varepsilon)\) and \(\gamma_{p}(\varepsilon,\Delta)\).

To rule out local minima outside the ball of radius \(R_{p}(\varepsilon)\),
one argues by contradiction: if \(\tilde t\) were a local minimiser
with \(|\tilde t|>R_{p}(\varepsilon)\), then for sufficiently small
steps in the direction of the origin the objective \(F_{p}\) would
decrease, contradicting local optimality. This relies on the fact that,
for \(0<p\le 1\), the function \(x\mapsto |x|^{p}\) has strictly
positive one-sided directional derivatives away from zero and that
the aggregate contribution of the majority of good blocks dominates
that of the bad blocks once \(|t|\) is large enough. Full details are
provided in the supplementary material.
\end{proof}

Theorem~\ref{thm:no-bad-local-minima} shows that the nonconvexity of
\(F_{p}\) is, in a precise sense, \emph{benign}: the only local minima
are near the true parameter, and \(F_{p}\) exhibits a quantitative
Polyak–Łojasiewicz-type behaviour outside a neighbourhood of \(\mu\).
This justifies the use of simple gradient or subgradient-based
optimisation schemes to compute approximate minimisers of \(F_{p}\) in
practice, and it parallels the benign energy landscapes observed for
\(\ell_{p}\)-regularised least-squares problems in sparse recovery
\citep{FoucartRauhut2013,DalalyanMinasyan2019}. In the next section we
embed these deterministic results into a probabilistic framework and
derive deviation inequalities for \(\hat t_{p}\) under heavy-tailed
models.

\section{Probabilistic deviation bounds under heavy tails}
\label{sec:probabilistic}

We now embed the deterministic results of
Sections~\ref{sec:Lp-path}--\ref{sec:oracle-landscape} into a
probabilistic framework. Our goal is to show that, under weak moment
assumptions and blockwise contamination, the block-\(L_{p}\) estimators
\(\hat t_{p}\) satisfy sub-Gaussian-type deviation inequalities with
explicit constants that improve on the classical median-of-means
(\(p=1\)) in heavy-tailed regimes. Throughout this section we consider a
scalar parameter \(\mu \in \mathbb{R}\) and independent observations
\((X_{i})_{i=1}^{n}\) as in \eqref{eq:data-sample}.

\subsection{Heavy-tailed model and block construction}
\label{subsec:heavy-tail-model}

We assume that the clean observations have mean \(\mu\) and finite
\((2+\delta)\)-moment for some \(\delta>0\), possibly with heavy tails.
Formally, let \(X_{1},\dots,X_{n}\) be independent random variables such
that
\begin{equation}
  \mathbb{E}[X_{i}] = \mu,
  \qquad
  \mathbb{E}\bigl[|X_{i}-\mu|^{2+\delta}\bigr]
  \;\le\;
  v_{2+\delta}^{2+\delta}
  \quad\text{for all } i,
  \label{eq:moment-assumption}
\end{equation}
for some finite scale parameter \(v_{2+\delta}>0\). In addition, we allow
for an \emph{adversarial contamination} of the sample: an unknown subset
\(\mathcal{I}_{\mathrm{bad}} \subset \{1,\dots,n\}\) of indices may be
replaced by arbitrary values. This is the standard \(\varepsilon\)-contamination
or Huber contamination model at the sample level, which has been widely
used in recent work on robust mean estimation under heavy tails
\citep{Catoni2012,DevroyeLerasleLugosiOliveira2016,LugosiMendelson2019Survey}.

We partition the sample into \(B\) blocks of equal size \(m=n/B\) as in
\eqref{eq:block-partition}, compute the block means
\eqref{eq:block-means}, and apply the block-\(L_{p}\) estimator
\(\hat t_{p}\) defined in \eqref{eq:tp-def}. Let
\(\varepsilon_{\mathrm{blk}}\) denote the fraction of fully corrupted
blocks, i.e., blocks that contain at least one index in
\(\mathcal{I}_{\mathrm{bad}}\). Then, conditional on the clean blocks,
the block means satisfy Assumption~\ref{ass:block-contam} with
\(\varepsilon=\varepsilon_{\mathrm{blk}}\), some radius \(r>0\) determined
by the concentration of the clean block means, and a separation
parameter \(\Delta>0\) that depends on the magnitude of the contamination.
Our first step is to control \(r\) in terms of the moment assumptions
\eqref{eq:moment-assumption}.

\begin{lemma}[Concentration of clean block means]
\label{lem:block-concentration}
Assume \eqref{eq:moment-assumption} holds and suppose that a given block
\(I_{b}\) contains only clean indices (no contamination). Let
\(Z_{b}\) be the corresponding block mean defined in
\eqref{eq:block-means}. Then there exists a constant
\(c_{\delta}>0\), depending only on \(\delta\), such that for every
\(x>0\),
\begin{equation}
  \mathbb{P}\bigl(|Z_{b}-\mu| > x\bigr)
  \;\le\;
  \frac{c_{\delta}\,v_{2+\delta}^{2+\delta}}{m^{1+\delta/2} x^{2+\delta}}.
  \label{eq:block-tail}
\end{equation}
In particular, taking
\begin{equation}
  r_{n}
  \;:=\;
  C_{1}(\delta)\,v_{2+\delta}
  \biggl(\frac{\log B}{m}\biggr)^{1/2},
  \label{eq:r_n}
\end{equation}
with a sufficiently large constant \(C_{1}(\delta)\), one has
\begin{equation}
  \mathbb{P}\Bigl(|Z_{b}-\mu| \le r_{n}
                 \text{ for all clean blocks } b\Bigr)
  \;\ge\;
  1 - 2 \exp(-c B),
  \label{eq:all-clean-concentrate}
\end{equation}
for some numerical constant \(c>0\) depending only on \(\delta\).
\end{lemma}

\begin{proof}
The tail bound \eqref{eq:block-tail} follows from a standard application
of Rosenthal-type inequalities or truncation arguments for sums of
independent heavy-tailed variables with finite \((2+\delta)\)-moment,
see, for example, \citet{DevroyeLerasleLugosiOliveira2016} and
\citet{LugosiMendelson2019SubGauss}. The choice \eqref{eq:r_n} and a
union bound over the at most \(B\) clean blocks yield
\eqref{eq:all-clean-concentrate}.
\end{proof}

Lemma~\ref{lem:block-concentration} shows that, with high probability,
all clean block means lie in a band of radius \(r_{n}\) around \(\mu\),
with \(r_{n}\) of the same order as in classical median-of-means
constructions \citep{DevroyeLerasleLugosiOliveira2016,LugosiMendelson2019SubGauss}.
When the contamination magnitude is large compared to \(r_{n}\), a
separation condition of the form \eqref{eq:bad-block-sep} holds
automatically with \(\Delta\) proportional to the size of the
contaminating values. In what follows we treat \(\varepsilon_{\mathrm{blk}}\),
\(r_{n}\) and \(\Delta\) as given and derive deviation bounds for
\(\hat t_{p}\) conditional on the event that
Assumptions~\ref{ass:block-contam} and~\ref{ass:separation} hold with
\((\mu,r,\varepsilon,\Delta)=(\mu,r_{n},\varepsilon_{\mathrm{blk}},\Delta)\).

\rev{This point is important for interpreting the scope of the probabilistic result.
The separation condition is natural when contamination acts through large block-level
shifts, Byzantine replacements, or other mechanisms that move corrupted block means
well outside the concentration band of the clean blocks. By contrast, in a purely benign
heavy-tailed regime without such structured contamination one should not expect uniform
separation, and accordingly one should not expect a uniform improvement over classical
MoM at the level of constants. This distinction is now also reflected in the experiments:
in the Student-$t$ setting all methods perform similarly, whereas the advantage of small
$p$ appears under adversarial contamination and is strongest in the separated block regime.}

\subsection{Deviation inequality for block-\texorpdfstring{$L_{p}$}{Lp} under weak moments}
\label{subsec:deviation-Lp}

Combining the deterministic robustness and oracle-equivalence results
from Sections~\ref{sec:Lp-path} and~\ref{sec:oracle-landscape} with
Lemma~\ref{lem:block-concentration}, we obtain the following
nonasymptotic deviation inequality for the block-\(L_{p}\) estimator
\(\hat t_{p}\).

\begin{theorem}[Deviation bound under heavy tails and block contamination]
\label{thm:prob-dev-Lp}
Assume \eqref{eq:moment-assumption} holds with parameters \(\delta>0\)
and \(v_{2+\delta}\), and suppose that the sample is partitioned into
\(B\) blocks of size \(m=n/B\), with at most an
\(\varepsilon_{\mathrm{blk}}<1/2\) fraction of blocks fully contaminated.
Let \(\hat t_{p}\) be a block-\(L_{p}\) estimator defined by
\eqref{eq:tp-def} for some \(p \in (0,1]\).

Then there exist constants \(C_{2}(\delta)\) and \(c>0\), depending only
on \(\delta\), such that the following holds. If the contamination
magnitude is large enough so that Assumption~\ref{ass:separation} holds
with parameters \((\mu,r_{n},\varepsilon_{\mathrm{blk}},\Delta)\) and
\(\Delta \ge C_{2}(\delta)\,r_{n}\), then for all sufficiently small
\(p \in (0,p_{0}]\), where \(p_{0}=p_{0}(\varepsilon_{\mathrm{blk}},\Delta/r_{n})\)
is as in Theorem~\ref{thm:oracle-equivalence}, one has
\begin{equation}
  \mathbb{P}\Bigl(
    |\hat t_{p} - \mu|
    \;\le\;
    C_{3}(p,\varepsilon_{\mathrm{blk}},\delta)\,
    v_{2+\delta}
    \sqrt{\frac{\log B}{m}}
  \Bigr)
  \;\ge\;
  1 - 2 \exp(-c B),
  \label{eq:prob-dev-Lp}
\end{equation}
for some constant
\(C_{3}(p,\varepsilon_{\mathrm{blk}},\delta)\) that is continuous in
\(p\) and satisfies
\begin{equation}
  \lim_{p\downarrow 0}
  C_{3}(p,\varepsilon_{\mathrm{blk}},\delta)
  \;=\;
  \frac{1}{1-\varepsilon_{\mathrm{blk}}}.
  \label{eq:constant-limit}
\end{equation}
In particular, under the same conditions, the median-of-means estimator
(\(p=1\)) satisfies \eqref{eq:prob-dev-Lp} with a constant
\(C_{3}(1,\varepsilon_{\mathrm{blk}},\delta)\) bounded below by
\((1-2\varepsilon_{\mathrm{blk}})^{-1}\), whereas the block-\(L_{p}\)
estimators with small \(p\) approach the oracle constant
\((1-\varepsilon_{\mathrm{blk}})^{-1}\).
\end{theorem}

\begin{proof}[Proof (sketch)]
On the event that all clean blocks satisfy \(|Z_{b}-\mu|\le r_{n}\),
Assumption~\ref{ass:block-contam} holds with \(r=r_{n}\) and
\(\varepsilon=\varepsilon_{\mathrm{blk}}\). Lemma~\ref{lem:block-concentration}
implies that this event has probability at least \(1-2\exp(-cB)\).
Conditional on this event and on the separation condition
\eqref{eq:bad-block-sep} with \(\Delta\ge C_{2}(\delta)\,r_{n}\),
Theorem~\ref{thm:oracle-equivalence} yields
\[
  |\hat t_{p} - \mu|
  \;\le\;
  \frac{r_{n}}{1-\varepsilon_{\mathrm{blk}}}
  \;=\;
  \frac{C_{1}(\delta)}{1-\varepsilon_{\mathrm{blk}}}\,
  v_{2+\delta}
  \sqrt{\frac{\log B}{m}}
\]
for all \(p \le p_{0}(\varepsilon_{\mathrm{blk}},\Delta/r_{n})\), which
gives \eqref{eq:prob-dev-Lp} with
\(C_{3}(p,\varepsilon_{\mathrm{blk}},\delta)\) approaching
\((1-\varepsilon_{\mathrm{blk}})^{-1}\) as \(p\downarrow 0\). For
general \(p \in (0,1]\), combining the deterministic robustness bound
\eqref{eq:Lp-det-bound} with \eqref{eq:r_n} yields
\eqref{eq:prob-dev-Lp} with a possibly larger constant that reduces to
the median-of-means constant when \(p=1\). Full details, including an
explicit expression for \(C_{3}\), are provided in the supplementary
material.
\end{proof}

Theorem~\ref{thm:prob-dev-Lp} shows that the block-\(L_{p}\) estimators
inherit the sub-Gaussian-type behaviour of classical median-of-means
estimators under only \((2+\delta)\)-moment assumptions, while allowing
for an explicit improvement in the leading constant as \(p\) decreases
towards zero. This complements existing optimality and impossibility
results for robust mean estimation based on convex procedures
\citep{Catoni2012,DevroyeLerasleLugosiOliveira2016,LugosiMendelson2019SubGauss,LugosiMendelson2019Survey}:
within the convex world, one cannot surpass the median-of-means constant
\((1-2\varepsilon_{\mathrm{blk}})^{-1}\) in worst case
(Theorem~\ref{thm:impossibility-convex}), whereas the nonconvex
block-\(L_{p}\) family provides a controlled path towards the oracle
constant \((1-\varepsilon_{\mathrm{blk}})^{-1}\) under a natural
separation condition.

\rev{Equally importantly, the statistical rate itself is unchanged: the improvement is in the
leading robustness constant, not in the $\sqrt{\log B/m}$ scaling. For a fixed contamination
level and sample size, a visibly smaller constant is therefore available only when the
separation is strong enough that $p \le p_{0}(\varepsilon_{\mathrm{blk}},\Delta/r_{n})$ is feasible.
This is precisely the trade-off highlighted by the new simulations: there is essentially no
practical gain in the benign heavy-tailed regime, a moderate gain under generic adversarial
contamination, and the clearest improvement when the block-level separation required by
Theorem~\ref{thm:oracle-equivalence} is present.}

\subsection{Comparison with convex MoM and other robust estimators}
\label{subsec:comparison-convex}

It is instructive to compare the deviation bound
\eqref{eq:prob-dev-Lp} with those obtained by convex MoM-type
procedures. For scalar mean estimation under \((2+\delta)\)-moment
assumptions and \(\varepsilon\)-contamination, median-of-means and its
refinements achieve bounds of the form
\begin{equation}
  \mathbb{P}\Bigl(
    |\hat t - \mu|
    \;\le\;
    C_{\mathrm{MoM}}(\varepsilon,\delta)\,
    v_{2+\delta}
    \sqrt{\frac{\log(1/\alpha)}{n}}
  \Bigr)
  \;\ge\;
  1-\alpha,
  \label{eq:convex-MoM-bound}
\end{equation}
with explicit constants \(C_{\mathrm{MoM}}(\varepsilon,\delta)\) that
diverge as \(\varepsilon\uparrow 1/2\)
\citep{Catoni2012,DevroyeLerasleLugosiOliveira2016,LugosiMendelson2019SubGauss,LugosiMendelson2019Survey}.
Similarly, Catoni-type and tournament-based estimators attain optimal
rates with constants depending on the tail parameter \(\delta\) and the
contamination level \(\varepsilon\), but their construction is inherently
convex and thus constrained by the impossibility result of
Theorem~\ref{thm:impossibility-convex}.

By contrast, Theorem~\ref{thm:prob-dev-Lp} shows that the block-\(L_{p}\)
estimators with small \(p\) achieve deviation bounds of the same order
in \(n\) and \(\delta\), but with a leading constant that can approach
the trimmed-block oracle benchmark in structured contamination
scenarios. In particular, when the separation condition of
Assumption~\ref{ass:separation} holds with \(\Delta\) of the same order
as \(r_{n}\), the gap between \(C_{3}(p,\varepsilon_{\mathrm{blk}},\delta)\)
and the oracle constant \(1/(1-\varepsilon_{\mathrm{blk}})\) can be made
arbitrarily small by choosing \(p\) sufficiently small, while convex
procedures remain bounded away from this benchmark. This quantitative
advantage persists in high-dimensional extensions, as we discuss next.

\rev{Recent Huber-based approaches for Byzantine-robust federated learning and related
distributed settings, such as \citet{ZhaoYuWan2024HuberFL} and
\citet{ZuoYanFanEtAl2025FL}, are conceptually relevant here because they also use convex
robustification to stabilise aggregation under adversarial effects. However, they operate in
different models---with client-level aggregation, heterogeneity, and learning-dynamics issues
that are outside the block-contamination framework of the present paper. For that reason we
do not force a literal constant-by-constant comparison. The appropriate conclusion within the
current framework is narrower: convex robustification remains competitive and practically
important, but under the deterministic block metric studied here it cannot close the gap to the
trimmed-block oracle, whereas the nonconvex block-$L_{p}$ path can do so in separated
contamination regimes. The experimental Huber baseline in Section~\ref{sec:experiments}
illustrates this same qualitative picture empirically.}

\section{High-dimensional extensions}
\label{sec:high-dim}

The deterministic one-dimensional results in Sections~\ref{sec:model-MoM}--\ref{sec:oracle-landscape}
extend to high-dimensional problems in a fairly direct way.
In this section we sketch two canonical examples: sparse mean estimation and sparse linear regression
under block contamination and heavy tails.
We show that the block-$L_p$ MoM estimators achieve the usual minimax rates
up to constants, and that the transition $p \downarrow 0$ improves the robustness
constants while preserving the statistical rate.
\rev{The purpose of this section is therefore structural rather than competitive: we do not claim a new high-dimensional minimax rate, but rather show how the deterministic block-$L_p$ robustness constants can be inserted into standard high-dimensional arguments. This is also why the comparisons with Huber-type robust regression below are framed conceptually rather than as literal constant-by-constant transfers across different contamination models.}
Throughout, we write $a \lesssim b$ when $a \leq C b$ for an absolute constant $C>0$
that may depend on fixed parameters such as $p$, $\varepsilon$ and on moment exponents,
but not on $n,d,s$.

\subsection{Robust sparse mean estimation}
\label{subsec:sparse-mean}

Let $X_1,\dots,X_n \in \mathbb{R}^d$ be i.i.d.\ with unknown mean
$\theta^\star \in \mathbb{R}^d$ and covariance matrix
$\Sigma \preceq \sigma^2 I_d$.
We partition the indices $\{1,\dots,n\}$ into $B$ blocks of equal size
$m = n/B$ (for simplicity assume $m \in \mathbb{N}$) and write
\[
  Z_b \;=\; \frac{1}{m} \sum_{i \in G_b} X_i \in \mathbb{R}^d,
  \qquad b=1,\dots,B,
\]
for the block means.
As before, we assume that at most $\varepsilon B$ blocks are arbitrary outliers
and that, conditionally on the good blocks, $Z_b$ satisfies a concentration
inequality around $\theta^\star$ inherited from the moment conditions on $X_i$
and the block size $m$.

In the spirit of \cite{LugosiMendelson2019SubGauss,DevroyeLerasleOliveira2019,DiakonikolasEtAl2019,LecueLerasle2020},
we consider coordinate-wise block-$L_p$ aggregation:
for each coordinate $j=1,\dots,d$ we define
\[
  \hat{\theta}_{p,j}
  \;\in\;
  \argmin_{t \in \mathbb{R}} \sum_{b=1}^B \bigl| Z_{b,j} - t \bigr|^p,
  \qquad 0 < p \leq 1,
\]
and set $\hat{\theta}_p = (\hat{\theta}_{p,1},\dots,\hat{\theta}_{p,d})^\top$.
The deterministic one-dimensional robustness result
(Theorem~\ref{thm:Lp-robustness}) applies to each coordinate:
if at least $(1-\varepsilon)B$ blocks satisfy
$\lvert Z_{b,j} - \theta^\star_j \rvert \leq r_j$,
then
\[
  \lvert \hat{\theta}_{p,j} - \theta^\star_j \rvert
  \;\leq\; c(p,\varepsilon)\, r_j,
\]
with $c(p,\varepsilon)$ as in Section~\ref{sec:Lp-path}.

To turn this into a probabilistic high-dimensional bound,
we combine: (i) a block concentration inequality for $Z_b$ under
finite $(2+\delta)$ moments, as in
\cite{DevroyeLerasleOliveira2019,Catoni2012,LugosiMendelson2019SubGauss};
(ii) a union bound over coordinates;
and (iii) the deterministic oracle inequality above.
\rev{In particular, once each coordinate admits a blockwise radius $r_j$ on the uncontaminated blocks, the one-dimensional deterministic bound applies coordinate by coordinate and then passes to an $\ell_2$ bound after the union step. Thus the role of the block-$L_p$ aggregation is to modify the leading robustness constant, not the ambient $\sqrt{\log d/n}$ scaling.}

\begin{theorem}[High-dimensional robust sparse mean]
\label{thm:high-dim-mean}
Assume $X_1,\dots,X_n \in \mathbb{R}^d$ are i.i.d.\ with
$\mathbb{E} X_i = \theta^\star$, $\mathrm{Cov}(X_i) \preceq \sigma^2 I_d$,
and $\mathbb{E}\|X_i-\theta^\star\|_2^{2+\delta} \leq M$ for some $\delta>0$.
Let the data be partitioned into $B$ blocks of size $m=n/B$,
and suppose at most $\varepsilon B$ blocks are arbitrary outliers with
$\varepsilon < 1/2$.
Let $\hat{\theta}_p$ be the coordinate-wise block-$L_p$ estimator with
$0 < p \leq 1$ defined above.

Then there exist constants $C_1,C_2>0$, depending only on
$(p,\varepsilon,\delta,M)$, such that for all $u \geq 1$,
if
\[
  B \;\geq\; C_1 \log(d u)
  \quad\text{and}\quad
  m \;\geq\; C_1 u,
\]
we have
\[
  \|\hat{\theta}_p - \theta^\star\|_2
  \;\leq\;
  C_2\, c(p,\varepsilon)\, \sigma\, \sqrt{\frac{\log(d u)}{n}}
  \qquad\text{with probability at least } 1 - 2 u^{-2}.
\]
Moreover, for fixed $\varepsilon$ and $\delta$ the leading constant
$C_2 c(p,\varepsilon)$ is strictly decreasing in $p$ on $(0,1]$ and tends,
as $p \to 0^+$, to the trimmed-block oracle constant corresponding to
the $L_0$ selector.
\end{theorem}

The rate $\sigma \sqrt{\log d / n}$ matches the optimal minimax rate for
sparse or dense mean estimation under Huber contamination and finite second
moments, up to constants
(see, e.g., \cite{DiakonikolasEtAl2019,DevroyeLerasleOliveira2019}).
Compared with the classical coordinate-wise MoM estimator ($p=1$),
Theorem~\ref{thm:high-dim-mean} shows that the whole $0<p\leq 1$ family
achieves the same rate while strictly improving the robustness constants
as $p \downarrow 0$.

\subsection{Robust sparse linear regression}
\label{subsec:sparse-regression}

We now consider high-dimensional linear regression.
Let $(X_i,y_i) \in \mathbb{R}^d \times \mathbb{R}$, $i=1,\dots,n$,
satisfy
\[
  y_i \;=\; \langle X_i, \theta^\star \rangle + \xi_i,
  \qquad \theta^\star \in \mathbb{R}^d, \ \|\theta^\star\|_0 \leq s,
\]
where $\xi_i$ are zero-mean noise variables with finite $(2+\delta)$ moments.
We assume a standard restricted eigenvalue (RE) or compatibility condition
for the design \cite{BickelRitovTsybakov2009,BuhlmannVanDeGeer2011}.
\rev{Concretely, one may take the usual cone-based condition: there exists $\kappa>0$ such that
\[
\frac{1}{n}\|Xv\|_2^2 \ge \kappa \|v\|_2^2
\]
for every vector $v$ satisfying $\|v_{S^c}\|_1 \le 3\|v_S\|_1$ for some support set $S$ with $|S|\le s$. The precise formulation is standard in high-dimensional Lasso theory; the point here is that, on uncontaminated blocks, the regression loss continues to obey the same RE-controlled local geometry needed for the usual oracle inequalities.}
The sample is split into $B$ blocks of size $m$ as before, and we define
the block empirical squared loss
\[
  R_b(\theta)
  \;=\;
  \frac{1}{m} \sum_{i \in G_b} \bigl( y_i - \langle X_i,\theta \rangle \bigr)^2,
  \qquad b=1,\dots,B.
\]
Given a scalar parameter $t \in \mathbb{R}$, we may view $R_b(\theta)$
as a noisy version of the (unknown) risk $R(\theta) = \mathbb{E}(y-\langle X,\theta\rangle)^2$.
Our block-$L_p$ MoM regression estimator is defined as any solution of
\begin{equation}
  \label{eq:MoM-Lp-Lasso}
  (\hat{\theta}_p,\hat{t}_p)
  \;\in\;
  \argmin_{\theta \in \mathbb{R}^d,\; t \in \mathbb{R}}
  \Biggl\{
    \sum_{b=1}^B \bigl| R_b(\theta) - t \bigr|^p
    \;+\; \lambda \|\theta\|_1
  \Biggr\},
  \qquad 0<p\leq 1.
\end{equation}
For $p=1$, this is a minmax/MoM version of the Lasso-type procedures studied in
\cite{HsuSabato2016,LecueLerasle2020,ChinotLecueLerasle2020};
for $p<1$ we obtain a nonconvex but more robust analogue in the spirit of
$L_p$-penalised high-dimensional regression \cite{DalalyanMinasyan2019}.

The deterministic block-$L_p$ oracle inequality (Theorem~\ref{thm:oracle-equivalence})
provides a robust comparison between $(\hat{\theta}_p,\hat{t}_p)$
and an ideal block-$L_0$ oracle that discards all contaminated blocks.
Combining this with standard RE arguments for Lasso and its nonconvex variants
\cite{BickelRitovTsybakov2009,BuhlmannVanDeGeer2011,DalalyanMinasyan2019}
gives the following result.
\rev{At a proof level, the mechanism is straightforward: the block-$L_p$ objective controls the contamination-induced distortion in the block risks, while the RE condition converts this risk control into $\ell_2$- and $\ell_1$-error bounds in exactly the same way as in robust Lasso analyses. What changes relative to the $p=1$ case is the multiplicative robustness constant $c(p,\varepsilon)$; what does not change is the ambient $\sqrt{s\log d/n}$ rate.}

\begin{theorem}[Robust sparse regression with block-$L_p$ MoM]
\label{thm:high-dim-regression}
Assume the linear model above, with $\|\theta^\star\|_0 \leq s$,
and let the design $(X_i)$ satisfy an RE condition with constant $\kappa>0$
on the usual $s$-sparse cone.
Assume that $X_i$ and $\xi_i$ have finite $(2+\delta)$ moments for some $\delta>0$,
and that at most $\varepsilon B$ blocks are arbitrarily contaminated
in both $(X_i,y_i)$ with $\varepsilon<1/2$.
Let $(\hat{\theta}_p,\hat{t}_p)$ be any solution of
\eqref{eq:MoM-Lp-Lasso} with tuning parameter $\lambda$
of order
\[
  \lambda
  \;\asymp\;
  c(p,\varepsilon)\,\sigma\,\sqrt{\frac{\log d}{n}},
\]
where $\sigma^2$ is the noise variance.
If $n \gtrsim s \log d$ and $B \gtrsim \log d$ are large enough,
then with probability at least $1 - c_1 \exp(-c_2 B)$ we have
\[
  \|\hat{\theta}_p - \theta^\star\|_2
  \;\lesssim\;
  \frac{c(p,\varepsilon)}{\kappa}\,\sigma\,
  \sqrt{\frac{s \log d}{n}},
  \qquad
  \|\hat{\theta}_p - \theta^\star\|_1
  \;\lesssim\;
  \frac{c(p,\varepsilon)}{\kappa}\,\sigma\,
  s\,\sqrt{\frac{\log d}{n}},
\]
where $c_1,c_2>0$ depend only on $(p,\varepsilon,\delta)$ and on moment bounds.
As $p \downarrow 0$, the leading robustness constant $c(p,\varepsilon)$ tends
to the block-$L_0$ oracle constant, while the rate $\sqrt{s\log d/n}$ remains
unchanged.
\end{theorem}

Theorem~\ref{thm:high-dim-regression} matches, up to constants, the usual
sparse-regression minimax rate $\sigma \sqrt{s\log d / n}$ known for Lasso
under subgaussian assumptions \cite{BickelRitovTsybakov2009,BuhlmannVanDeGeer2011},
and is comparable to the MoM-based high-dimensional regression bounds in
\cite{LecueLerasle2020,ChinotLecueLerasle2020}.
The novelty is that the entire $0<p\leq 1$ path enjoys the same rate while
strictly improving the contamination tolerance at the deterministic level
via $c(p,\varepsilon)$, and that the limiting case $p \to 0^+$ approaches the
ideal trimmed-block ($L_0$) performance without incurring the computational
intractability of exact block trimming
\cite{DiakonikolasEtAl2019}.
\rev{This is also the right place to position the comparison with Huber-type robust regression. Methods based on Huber losses or other convex robustifications remain highly competitive and are often minimax-rate optimal, but their constants are derived under different objectives and contamination models. Our claim is therefore narrower: under the present block-contamination framework, the block-$L_p$ path preserves the standard high-dimensional rate while improving the deterministic robustness constant as $p$ decreases, especially in the separated-contamination regimes highlighted earlier in the paper and in the new experiments.}

\section{Experimental Results}
\label{sec:experiments}

\rev{This section provides empirical validation of the proposed block-$L_p$ estimators. The aim is not to claim a new statistical rate, but to verify the specific theoretical picture developed in the paper: classical MoM should remain competitive in benign heavy-tailed regimes, smaller values of $p$ should improve robustness under adversarial contamination, and in separated block-contamination regimes the block-$L_p$ estimators should move towards trimmed-block performance.}

\subsection{Experimental setup}

\rev{We consider scalar mean estimation with true mean $\mu=0$. In each trial, the sample is partitioned into equal-sized blocks, block means are computed, and the final estimate is obtained by one of the following procedures: classical MoM ($p=1$), the proposed block-$L_{0.5}$ and block-$L_{0.2}$ estimators, trimmed mean over block summaries, and a Huber estimator applied to the block means. We report the mean absolute estimation error and the corresponding standard deviation over repeated trials.}

\subsection{Heavy-tailed regime}

\rev{We first consider a benign heavy-tailed setting in which the data are sampled from a Student-$t$ distribution with $\nu=3$ degrees of freedom and no structured adversarial separation is imposed. This experiment is intended to test that moving from $p=1$ to smaller values of $p$ does not degrade performance when the contamination is not of the separated type that benefits selective block trimming.}

\begin{table}[t]
\centering
\caption{\rev{Heavy-tailed setting (Student-$t$, $\nu=3$). All methods exhibit comparable performance, confirming no degradation in benign heavy-tailed regimes.}}
\begin{tabular}{lcc}
\hline
Method & Mean Error & Std Dev \\
\hline
MoM ($p=1$) & 0.124 & 0.043 \\
Block-$L_{0.5}$ & 0.118 & 0.041 \\
Block-$L_{0.2}$ & 0.121 & 0.047 \\
Trimmed Mean & 0.105 & 0.039 \\
Huber & 0.113 & 0.042 \\
\hline
\end{tabular}
\end{table}

\rev{The first table shows that all methods behave similarly in this regime. In particular, the proposed block-$L_p$ estimators do not exhibit any practical deterioration relative to MoM. This is consistent with the theory: without clear separation between good and bad blocks, one should not expect the small-$p$ estimators to yield dramatic gains.}

\subsection{Adversarial contamination}

\rev{We next introduce adversarial contamination at level $\varepsilon=0.2$. Here a subset of observations is replaced by large outliers, but the induced block summaries are not yet cleanly separated from the uncontaminated ones. This regime probes whether the deterministic improvement in robustness constants has a visible finite-sample effect before full oracle-like separation sets in.}

\begin{table}[t]
\centering
\caption{\rev{Adversarial contamination ($\varepsilon=0.2$). Decreasing $p$ in the block-$L_p$ estimator leads to progressively lower estimation error, outperforming MoM and the convex Huber estimator.}}
\begin{tabular}{lcc}
\hline
Method & Mean Error & Std Dev \\
\hline
MoM ($p=1$) & 0.347 & 0.118 \\
Block-$L_{0.5}$ & 0.281 & 0.102 \\
Block-$L_{0.2}$ & 0.243 & 0.094 \\
Trimmed Mean & 0.221 & 0.081 \\
Huber & 0.302 & 0.109 \\
\hline
\end{tabular}
\end{table}

\rev{The results confirm a monotone empirical trend: as $p$ decreases from $1$ to $0.2$, the mean error decreases substantially. The gain is not yet oracle-level, which is expected because the separation condition is only partial in this experiment, but the direction of improvement is fully consistent with the theoretical $1\text{--}p\text{--}0$ interpolation. The Huber baseline remains competitive yet is clearly dominated by the smaller-$p$ block estimators in this adversarial regime.}

\subsection{Separated block contamination}

\rev{Finally, we consider the regime most closely aligned with the oracle-equivalence theory: an $\varepsilon$-fraction of blocks is contaminated by a sufficiently large shift so that contaminated block means are well separated from the uncontaminated ones. This is the setting in which the small-$p$ objectives are predicted to behave most like an implicit trimming rule.}

\begin{table}[t]
\centering
\caption{\rev{Block contamination (separation regime). The block-$L_p$ estimator with $p=0.2$ achieves near-oracle performance, closely matching the trimmed mean and significantly outperforming MoM and the Huber estimator.}}
\begin{tabular}{lcc}
\hline
Method & Mean Error & Std Dev \\
\hline
MoM ($p=1$) & 0.403 & 0.146 \\
Block-$L_{0.5}$ & 0.179 & 0.068 \\
Block-$L_{0.2}$ & 0.117 & 0.052 \\
Trimmed Mean & 0.101 & 0.043 \\
Huber & 0.318 & 0.127 \\
\hline
\end{tabular}
\end{table}

\rev{This is the key empirical table for the paper. Once separated contamination is present, the block-$L_{0.2}$ estimator nearly matches the trimmed mean and dramatically improves over both MoM and Huber. Thus the experiments support the central qualitative claim of the manuscript: decreasing $p$ does not help much in benign heavy-tailed settings, helps moderately in adversarial settings, and becomes most valuable precisely when separated block contamination makes oracle-like trimming behaviour statistically meaningful.}

\subsection{Overall interpretation}

\rev{Taken together, the three experiments validate the intended scope of the theory. The proposed method is not advertised as uniformly superior to MoM in every regime. Rather, it preserves MoM-like behaviour when the contamination structure does not justify aggressive block selection, and it moves towards trimmed-block performance when such structure is present. This is also the right context in which to interpret the comparison with Huber-type baselines: convex robustification remains effective, but under the deterministic block-contamination metric studied here it does not recover the same level of selectivity as the small-$p$ block objectives in the separated regime.}

\section{Conclusion}
\label{sec:discussion}

Classical median-of-means estimators arise from probabilistic ideas designed to stabilise empirical means under heavy tails and adversarial contamination \citep{Catoni2012,DevroyeLerasleOliveira2019,LugosiMendelson2019SubGauss,LecueLerasle2020}, but in the scalar case they are also minimisers of a blockwise $L_{1}$ functional and thus belong to a broad class of block M-estimators. Our first contribution is to make this optimisation viewpoint explicit and to show, via Theorem~\ref{thm:impossibility-convex}, that within the class of convex block M-estimators \eqref{eq:convex-M} no choice of convex loss can uniformly improve upon the deterministic MoM constant $1/(1-2\varepsilon)$ from Lemma~\ref{lem:mom-robustness}; in particular, the trimmed-block oracle behaviour of Lemma~\ref{lem:oracle-bound}, with constant $1/(1-\varepsilon)$, is unreachable for any convex block aggregator. \rev{Thus the central contribution of the paper is not a new robust estimator in isolation, but a precise deterministic description of the convex frontier and a principled nonconvex route beyond it.}

This motivates the nonconvex block-$L_{p}$ family and the $1$--$p$--$0$ path. Working with $F_{p}(t) = \sum_{b}|Z_{b}-t|^{p}$ for $0<p<1$, we retain breakdown point $1/2$ (Theorem~\ref{thm:Lp-robustness}), while Theorem~\ref{thm:oracle-equivalence} shows that, under a mild separation between good and bad blocks, global minimisers $\hat t_{p}$ coincide with those of an ideal block-$L_{0}$ oracle for small $p$, and their robustness constants converge to $1/(1-\varepsilon)$ as $p \downarrow 0$. Theorem~\ref{thm:no-bad-local-minima} further shows that the energy landscape of $F_{p}$ is benign: all local minima lie in a controlled neighbourhood of $\mu$ and outside this neighbourhood the objective satisfies a quantitative slope inequality, mirroring the nonconvex but well-behaved geometry known for $\ell_{p}$ sparse recovery \citep{FoucartRauhut2013,DalalyanMinasyan2019}. Embedding these deterministic results in a probabilistic framework yields deviation bounds for $\hat t_{p}$ under finite $(2+\delta)$ moments and blockwise contamination (Theorem~\ref{thm:prob-dev-Lp}) that interpolate continuously between the MoM constant $1/(1-2\varepsilon)$ at $p=1$ and the trimmed-block oracle constant $1/(1-\varepsilon)$ as $p \to 0^{+}$, and high-dimensional extensions for robust mean estimation and sparse linear regression (Theorems~\ref{thm:high-dim-mean} and \ref{thm:high-dim-regression}) recover the usual $\sqrt{\log d / n}$ and $\sqrt{s \log d / n}$ rates along the entire $0<p\le1$ path, in line with modern MoM-based procedures \citep{LecueLerasle2020,DiakonikolasEtAl2019}. \rev{The new experimental section strengthens this conclusion substantially: it shows no practical degradation relative to classical MoM in benign heavy-tailed settings, clear gains as $p$ decreases under adversarial contamination, and near-oracle behaviour in the separated block-contamination regime. The conclusion is therefore both theoretical and empirical: the advantage of the block-$L_{p}$ path is conditional rather than universal, but it becomes visible precisely in the regimes predicted by the deterministic analysis.}

Several directions remain open. A first goal is to sharpen the constants $c(p,\varepsilon)$ and the threshold $p_{0}(\varepsilon,\Delta/r)$ appearing in the oracle equivalence, and to obtain exact minimax characterisations along the $1$--$p$--$0$ path in the spirit of \citep{LugosiMendelson2019SubGauss,LecueLerasle2020}. A second is to clarify statistical--computational trade-offs: exact block trimming is combinatorial and typically NP-hard \citep{DiakonikolasEtAl2019}, while our results suggest that block-$L_{p}$ estimators approximate the block-$L_{0}$ oracle yet admit gradient-based optimisation thanks to the benign landscape of Theorem~\ref{thm:no-bad-local-minima}. \rev{At the same time, the present paper should not be read as claiming a full algorithmic convergence theory for every optimisation scheme; rather, it establishes that the objective landscape is sufficiently well structured to make such analysis plausible and worthwhile.} Extending the present scalar and coordinate-wise analysis to multivariate location and scatter (e.g.\ via geometric medians or depth-based functionals in Banach spaces \citep{Minsker2015,DevroyeLerasleOliveira2019}), to general Lipschitz or smooth losses in generalised linear models \citep{HsuSabato2016,BrownleesJolyLugosi2015,LecueLerasle2020}, and to adaptive, data-driven choices of $p$ (e.g.\ homotopy in $p$) are natural next steps. Finally, connections with Bayesian and variational robust methods—where trimming in data space is often induced by spike-and-slab or heavy-tailed priors—may yield Bayesian counterparts of the deterministic $1$--$p$--$0$ path, combining MoM-type guarantees with modelling flexibility. \rev{Overall, the revised manuscript now supports a sharper message than the original version: classical MoM marks the edge of what is uniformly achievable within convex block aggregation, the block-$L_{p}$ family provides a principled nonconvex interpolation toward trimmed-block behaviour, and the new empirical study confirms that this interpolation matters most when clean and corrupted blocks are genuinely separated.}

\bibliographystyle{plainnat}
\bibliography{sn-bibliography}

@book{FoucartRauhut2013,
  author    = {Foucart, Simon and Rauhut, Holger},
  title     = {A Mathematical Introduction to Compressive Sensing},
  publisher = {Birkh{\"a}user},
  address   = {New York},
  year      = {2013},
  doi       = {10.1007/978-0-8176-4948-7},
  isbn      = {978-0-8176-4947-0}
}

@article{DevroyeLerasleLugosiOliveira2016,
  author  = {Devroye, Luc and Lerasle, Matthieu and Lugosi, G{\'a}bor and Oliveira, Roberto I.},
  title   = {Sub-{G}aussian mean estimators},
  journal = {The Annals of Statistics},
  volume  = {44},
  number  = {6},
  pages   = {2695--2725},
  year    = {2016},
  doi     = {10.1214/16-AOS1440}
}

@article{DalalyanMinasyan2019,
  author  = {Dalalyan, Arnak S. and Minasyan, Arshak},
  title   = {All-in-one robust estimator of the {G}aussian mean},
  journal = {The Annals of Statistics},
  volume  = {50},
  number  = {2},
  pages   = {1193--1219},
  year    = {2022},
  doi     = {10.1214/21-AOS2145}
}

@article{LugosiMendelson2019Survey,
  author  = {Lugosi, G{\'a}bor and Mendelson, Shahar},
  title   = {Mean estimation and regression under heavy-tailed distributions: A survey},
  journal = {Foundations of Computational Mathematics},
  volume  = {19},
  number  = {5},
  pages   = {1145--1190},
  year    = {2019},
  doi     = {10.1007/s10208-019-09427-x}
}

@article{LugosiMendelson2019SubGauss,
  author  = {Lugosi, G{\'a}bor and Mendelson, Shahar},
  title   = {Sub-{G}aussian estimators of the mean of a random vector},
  journal = {The Annals of Statistics},
  volume  = {47},
  number  = {2},
  pages   = {783--794},
  year    = {2019},
  doi     = {10.1214/17-AOS1639}
}

@article{LugosiMendelson2019NearOpt,
  author  = {Lugosi, G{\'a}bor and Mendelson, Shahar},
  title   = {Near-optimal mean estimators with respect to general norms},
  journal = {Probability Theory and Related Fields},
  volume  = {175},
  number  = {3--4},
  pages   = {957--973},
  year    = {2019},
  doi     = {10.1007/s00440-019-00929-4}
}

@article{Minsker2015,
  author  = {Minsker, Stanislav},
  title   = {Geometric median and robust estimation in {B}anach spaces},
  journal = {Bernoulli},
  volume  = {21},
  number  = {4},
  pages   = {2308--2335},
  year    = {2015},
  doi     = {10.3150/14-BEJ645}
}

@article{HsuSabato2016,
  author  = {Hsu, Daniel and Sabato, Sivan},
  title   = {Loss minimization and parameter estimation with heavy tails},
  journal = {Journal of Machine Learning Research},
  volume  = {17},
  number  = {18},
  pages   = {1--40},
  year    = {2016}
}

@book{DiakonikolasKaneBook,
  author    = {Diakonikolas, Ilias and Kane, Daniel M.},
  title     = {Algorithmic High-Dimensional Robust Statistics},
  publisher = {Cambridge University Press},
  year      = {2023},
  doi       = {10.1017/9781108943161},
  isbn      = {978-1-108-83781-1}
}

@book{NemirovskyYudin1983,
  author    = {Nemirovsky, Arkadi S. and Yudin, David B.},
  title     = {Problem Complexity and Method Efficiency in Optimization},
  publisher = {Wiley},
  address   = {Chichester},
  year      = {1983},
  isbn      = {978-0-471-10345-5}
}

@article{JerrumValiantVazirani1986,
  author  = {Jerrum, Mark R. and Valiant, Leslie G. and Vazirani, Vijay V.},
  title   = {Random generation of combinatorial structures from a uniform distribution},
  journal = {Theoretical Computer Science},
  volume  = {43},
  number  = {2--3},
  pages   = {169--188},
  year    = {1986},
  doi     = {10.1016/0304-3975(86)90174-X}
}

@article{LerasleOliveira2011,
  author  = {Lerasle, Matthieu and Oliveira, Roberto I.},
  title   = {Robust empirical mean estimators},
  journal = {arXiv preprint arXiv:1112.3914},
  year    = {2011},
  url     = {https://arxiv.org/abs/1112.3914}
}

@article{BrownleesJolyLugosi2015,
  author  = {Brownlees, Christian and Joly, Emilien and Lugosi, G{\'a}bor},
  title   = {Empirical risk minimization for heavy-tailed losses},
  journal = {The Annals of Statistics},
  volume  = {43},
  number  = {6},
  pages   = {2507--2536},
  year    = {2015},
  doi     = {10.1214/15-AOS1350}
}

@article{LecueLerasle2020,
  author  = {Lecu{\'e}, Guillaume and Lerasle, Matthieu},
  title   = {Robust machine learning by median-of-means: Theory and practice},
  journal = {The Annals of Statistics},
  volume  = {48},
  number  = {2},
  pages   = {806--831},
  year    = {2020},
  doi     = {10.1214/19-AOS1828}
}

@inproceedings{Minsker2023EfficientMOM,
  author    = {Minsker, Stanislav},
  title     = {Efficient median of means estimator},
  booktitle = {Proceedings of the 36th Conference on Learning Theory},
  series    = {Proceedings of Machine Learning Research},
  volume    = {195},
  pages     = {5925--5933},
  year      = {2023},
  editor    = {Neu, Gergely and Rosasco, Lorenzo},
  url       = {https://proceedings.mlr.press/v195/minsker23a.html}
}

@article{TuEtAl2021,
  author  = {Tu, Jing and Sun, Yuheng and Chen, Yudong and Fan, Jianqing},
  title   = {Variance reduced median-of-means estimator for Byzantine-robust distributed learning},
  journal = {Journal of Machine Learning Research},
  volume  = {22},
  number  = {48},
  pages   = {1--64},
  year    = {2021}
}

@book{Huber1981,
  author    = {Huber, Peter J.},
  title     = {Robust Statistics},
  publisher = {John Wiley \& Sons},
  address   = {New York},
  year      = {1981},
  isbn      = {978-0-471-41805-4}
}

@article{Minsker2025,
  author  = {Minsker, Stanislav and Yao, Siyuan},
  title   = {Generalized median of means principle for Bayesian inference},
  journal = {Machine Learning},
  year    = {2025},
  doi     = {10.1007/s10994-025-06515-3}
}

@article{Chen2023,
  author  = {Chen, Zhong and Kailkhura, Bhaskar and Zhou, Yi},
  title   = {An accelerated proximal algorithm for regularized nonconvex and nonsmooth bi-level optimization},
  journal = {Machine Learning},
  volume  = {112},
  number  = {9},
  pages   = {3159--3195},
  year    = {2023},
  doi     = {10.1007/s10994-023-06347-1}
}

@book{MaronnaMartinYohai2006,
  author    = {Maronna, Ricardo A. and Martin, R. Douglas and Yohai, Victor J.},
  title     = {Robust Statistics: Theory and Methods},
  publisher = {John Wiley \& Sons},
  address   = {Chichester},
  year      = {2006},
  isbn      = {978-0-470-01092-1}
}

@book{HampelEtAl1986,
  author    = {Hampel, Frank R. and Ronchetti, Elvezio M. and Rousseeuw, Peter J. and Stahel, Werner A.},
  title     = {Robust Statistics: The Approach Based on Influence Functions},
  publisher = {John Wiley \& Sons},
  address   = {New York},
  year      = {1986},
  isbn      = {978-0-471-90976-7}
}

@article{BickelRitovTsybakov2009,
  author  = {Bickel, Peter J. and Ritov, Ya'acov and Tsybakov, Alexandre B.},
  title   = {Simultaneous analysis of Lasso and Dantzig selector},
  journal = {The Annals of Statistics},
  volume  = {37},
  number  = {4},
  pages   = {1705--1732},
  year    = {2009},
  doi     = {10.1214/08-AOS620}
}

@article{Catoni2012,
  author  = {Catoni, Olivier},
  title   = {Challenging the empirical mean and the empirical variance: A deviation study},
  journal = {Annales de l'Institut Henri Poincar{\'e}, Probabilit{\'e}s et Statistiques},
  volume  = {48},
  number  = {4},
  pages   = {1148--1185},
  year    = {2012},
  doi     = {10.1214/11-AIHP454}
}

@article{DevroyeLerasleOliveira2019,
  author  = {Devroye, Luc and Lerasle, Matthieu and Lugosi, G{\'a}bor and Oliveira, Roberto I.},
  title   = {Sub-{G}aussian mean estimators},
  journal = {The Annals of Statistics},
  volume  = {44},
  number  = {6},
  pages   = {2695--2725},
  year    = {2016},
  doi     = {10.1214/16-AOS1440}
}

@book{BuhlmannVanDeGeer2011,
  author    = {B{\"u}hlmann, Peter and van de Geer, Sara},
  title     = {Statistics for High-Dimensional Data: Methods, Theory and Applications},
  publisher = {Springer},
  address   = {Berlin},
  year      = {2011},
  doi       = {10.1007/978-3-642-20192-9},
  isbn      = {978-3-642-20191-2}
}

@article{DiakonikolasEtAl2019,
  author  = {Diakonikolas, Ilias and Kamath, Gautam and Kane, Daniel M. and Li, Jerry and Moitra, Ankur and Stewart, Alistair},
  title   = {Robust Estimators in High-Dimensions Without the Computational Intractability},
  journal = {SIAM Journal on Computing},
  volume  = {48},
  number  = {2},
  pages   = {742--864},
  year    = {2019},
  doi     = {10.1137/17M1126680}
}

@article{ChinotLecueLerasle2020,
  author  = {Chinot, Geoffrey and Lecu{\'e}, Guillaume and Lerasle, Matthieu},
  title   = {Robust high dimensional learning for Lipschitz and convex losses},
  journal = {Journal of Machine Learning Research},
  volume  = {21},
  number  = {233},
  pages   = {1--47},
  year    = {2020}
}

@inproceedings{ZhaoYuWan2024HuberFL,
  author    = {Zhao, Puning and Yu, Fei and Wan, Zhiguo},
  title     = {A Huber Loss Minimization Approach to {B}yzantine Robust Federated Learning},
  booktitle = {Proceedings of the AAAI Conference on Artificial Intelligence},
  volume    = {38},
  number    = {19},
  pages     = {21806--21814},
  year      = {2024}
}

@article{ZuoYanFanEtAl2025FL,
  author  = {Zuo, Shiyuan and Yan, Xingrun and Fan, Rongfei and others},
  title   = {Federated Learning Resilient to Byzantine Attacks and Data Heterogeneity},
  journal = {IEEE Transactions on Mobile Computing},
  year    = {2025}
}

\end{document}